\documentclass[onefignum,onetabnum]{siamonline190516}
\usepackage{amsfonts}
\usepackage{multirow}
\usepackage{algorithm,algpseudocode}

\newcommand{\Hc}{\mathcal{H}}

\newcommand{\Xc}{\mathcal{X}}
\newcommand{\Yc}{\mathcal{Y}}
\newcommand{\Zc}{\mathcal{Z}}

\newcommand{\rb}{{\mathbf r}}

\newcommand{\Dc}{{{\mathcal D}}}

\newcommand{\Pc}{{{\mathcal P}}}

\newcommand{\Fc}{{\mathcal F}}

\newcommand{\Rd}{{\mathbb R}}

\newcommand{\Md}{\mathbb{M}}
\newcommand{\Kd}{\mathbb{K}}

\newcommand{\Dd}{\mathbb{D}}

\headers{Optimal Transport driven CycleGAN}{B. Sim, G. Oh,  J. Kim, C. Jung, JC.Ye}
\title{Optimal Transport driven CycleGAN for Unsupervised Learning in  Inverse Problems
\thanks{Submitted on \today. \funding{This work was  supported by the National Research Foundation (NRF) of Korea  grant  NRF-2020R1A2B5B03001980.}}}

\author{
  Byeongsu Sim\thanks{Department of Mathematical Sciences, KAIST, Daejeon, Republic of Korea.}
    \and
   Gyutaek Oh\thanks{Department of Bio and Brain Engineering, KAIST, Daejeon, Republic of Korea.}
  \and
  Jeongsol Kim\footnotemark[3]
 \and
   Chanyong Jung\footnotemark[3]
\and
   Jong Chul Ye\thanks{Jong Chul Ye is with the Department of Bio and Brain Engineering and  Department of Mathematical Sciences, KAIST, Daejeon, Republic of Korea. \email{jong.ye@kaist.ac.kr} }
}

\begin{document}

\maketitle

\begin{abstract}
To improve the performance of  classical generative adversarial network (GAN), Wasserstein generative adversarial networks (W-GAN) was developed as a Kantorovich dual formulation of the optimal transport (OT) problem using Wasserstein-1 distance. However, it was not clear how cycleGAN-type generative models can be derived from the optimal transport theory.
Here we show that a novel  cycleGAN architecture can be derived as a Kantorovich dual OT formulation
if a penalized least square (PLS) cost with deep learning-based inverse path penalty is used as a transportation cost.
One of the most important advantages of this formulation is that
depending on the knowledge of the forward problem, distinct variations of cycleGAN architecture can be derived: for example,  one with two
pairs of generators and discriminators, and the other with only a single pair of generator and discriminator. Even for the two generator cases,
we show that the structural knowledge of the forward operator can lead to a simpler generator architecture which significantly
simplifies the neural network training. The new cycleGAN formulation, what we call the OT-cycleGAN, have been applied
 for various biomedical imaging problems,  such as
accelerated magnetic resonance imaging (MRI), super-resolution microscopy, and low-dose x-ray computed tomography (CT). Experimental results confirm the efficacy and flexibility of the theory.
\end{abstract}

\section{Introduction}

Inverse problems are ubiquitous  in   imaging \cite{bertero1998introduction}, computer vision \cite{forsyth2002computer},
and science \cite{chadan2012inverse}.
In inverse problems, a noisy measurement $y\in \Yc$ from an unobserved image $x\in \Xc$ is modeled by
\begin{eqnarray}
y&=&\Hc x + w \ , 
\end{eqnarray}
where $w$ is the  noise, and $\Hc:\Xc\mapsto \Yc$ is  the measurement operator. 
In inverse problems originating from physics, the measurement operator is usually represented by an integral equation \cite{engl1996regularization}:
\begin{eqnarray}\label{eq:g} 
\Hc x(\rb) &:=& \int_{\mathbb{R}^{d}}h(\rb,\rb')x(\rb')d\rb',
\quad \rb\in \Dc\subset \Rd^d,  
\quad d=2,3,
\label{Eq1}
\end{eqnarray}
where $h(\rb,\rb')$ is an integral kernel. 
Then, the inverse problem is formulated as an  estimation problem of the unknown  $x$ from the measurement $y$.
Unfortunately, the forward operator $\Hc$ often shrinks or sometimes eliminates some signals necessary to recover $x$ \cite{bertero1998introduction}. This nature of the operator makes the inverse problem ill-posed, that is, the recovery of $x$ is very sensitive to noise $w$ or there may be infinitely many possible $x$'s.

A classical strategy to mitigate the ill-posedness is  the  penalized least squares (PLS) approach:
\begin{eqnarray}\label{eq:problem}
\hat x =\arg\min_{x} c(x;y):= \|y -\Hc x\|^q+  R(x)
\end{eqnarray}
for  $q\geq 1$, where 
$R(x)$ is a regularization (or penalty) function ($l_1$, total variation (TV), etc.) \cite{chaudhuri2014blind,sarder2006deconvolution,mcnally1999three}.
In some inverse problems, the measurement operator $\Hc$ is not 
known,  so both the unknown operator $\Hc$  and the image $x$ should be estimated.

One of the limitations of the PLS formulation  is that the optimization problem should be solved again whenever new measurement comes.
On the other hand,  recent deep learning approaches for inverse problems  \cite{kang2017deep,hammernik2018learning} can quickly  produce
reconstruction results 
by inductively learning the inverse mapping between the input and matched label data in a supervised manner.
Therefore, the deep learning approaches are ideally suitable for inverse problems where
   fast reconstruction is necessary.
Unfortunately, in many applications,  matched label data are often difficult to obtain,
so the need for unsupervised learning is increasing.

So far, one of successful unsupervised learning methods for inverse problems is using the
 cycle consistent generative adversarial network (cycleGAN)\cite{kang2019cycle,lu2017conditional}.
 By imposing the cycle-consistency, artificial features due to mode collapsing behavior of
 generative adversarial networks (GAN) \cite{goodfellow2014generative})  can be mitigated
in the cycleGAN.
However, as the cycleGAN was originally proposed for image translation task, 
it is not clear how the imaging physics can be incorporated within the cycleGAN framework.
Therefore, one of the most important contributions of this paper is to show that a novel  cycleGAN architecture called OT-CycleGAN can be derived as a Kantorovich dual  formulation
if a penalized least square (PLS) cost composed of physics-based data fidelity  and deep learning-based inverse path penalty is used as a transportation cost.
In fact, this is a direct extension of Wasserstein-GAN (W-GAN) \cite{arjovsky2017wasserstein} that was derived as a dual formulation of OT problem with Wasserstein-1 distance.
Specifically, by additionally adding the data fidelity term, our OT-cycleGAN can be derived.
Our derivation is so general that 
various forms of cycleGAN architectures can be derived as special cases of OT-cycleGAN depending on the amount of knowledge of the forward physics.

This paper is structured as follows. We review
the mathematical preliminaries 
and discuss the limitation of naive application of cycleGAN in Section \ref{sec:review}.
In Section \ref{sec:main}, we first propose a novel PLS formulation that incorporates an inverse path described by a neural network as a regularization term.
Then, we derive a general form of OT-cycleGAN by analyzing the adopted OT problem associated with the novel PLS cost, 
and discuss its three distinct variations.
As proofs of concept,  Section~\ref{sec:exp}  applies three distinct cycleGAN architectures for unsupervised learning in three physical inverse problems: accelerated MRI,   super-resolution microscopy, and low-dose CT reconstruction.
The experimental results confirm that the proposed unsupervised learning approaches can successfully provide accurate inversion results
without any matched reference.

\section{Related Works}
\label{sec:review}

\subsection{Optimal Transport (OT)}

Optimal transport provides a mathematical means to compare two probability measures \cite{villani2008optimal,peyre2019computational}. 
 Formally, we say that  $T:\Xc \mapsto \Yc$ transports the probability measure $\mu \in P(\Xc)$ to another measure $\nu \in  P(\Yc)$, if
\begin{eqnarray}\label{eq:constraint}
\nu(B) = \mu\left(T^{-1}(B)\right),\quad \mbox{for all $\nu$-measurable sets $B$},
\end{eqnarray}
which is often simply represented by
$\nu = T_\#\mu,$
where $T_\#$ is called the push-forward operator.
Suppose there is a cost function $c:\Xc \times \Yc \rightarrow \Rd\cup\{\infty\}$ such that $c(x,y)$ represents the cost of moving one unit of mass from $x \in \Xc$ to $y \in \Yc$.
Monge's original OT problem \cite{villani2008optimal,peyre2019computational} is then to find a transport map $T$ that transports $\mu$ to $\nu$
at the minimum total transportation cost:
\begin{eqnarray}\label{eq:monge}
\min_T &~\Md(T):=\int_\Xc c(x,T(x))d\mu(x) ,&\quad
\mbox{subject to}  \quad \nu = T_\#\mu  \notag .
\end{eqnarray}
The nonlinear constraint $\nu = T_\# \mu$ is difficult to handle and sometimes leads void $T$ due to assignment of indivisible mass \cite{villani2008optimal,peyre2019computational}.
Kantorovich  relaxed the assumption to consider probabilistic
transport that allows  mass splitting from a source
toward several targets. 
Specifically, Kantorovich  introduced  a joint measure $\pi \in P(\Xc\times \Yc)$ such that
the original problem can be relaxed as
\begin{eqnarray}\label{eq:kantorovich}
\min_\pi &~\int_{\Xc\times \Yc} c(x,y)d\pi(x,y) \\
\mbox{subject to}\quad & \pi(A\times \Yc)=\mu(A),~~
 \pi(\Xc\times B)=\nu(B) \notag
\end{eqnarray}
for all measurable sets $A\in \Xc$ and $B\in \Yc$.
Here,  the last two constraints come from the observation that the total amount of mass removed from any measurable set  has to be equal to the marginals
 \cite{villani2008optimal,peyre2019computational}.

One of the most important advantages of Kantorovich formulation is  the dual formulation as stated in the following theorem:
\begin{theorem}[Kantorovich duality theorem]\cite[Theorem 5.10, p.57-p.59]{villani2008optimal}
Let $(\Xc, \mu)$ and $(\Yc, \nu)$ be two Polish probability spaces (separable complete metric space) and let $c: \Xc \times \Yc \rightarrow \mathbb{R}$ be a continuous cost function, such that $|c(x,y)| \le c_{\Xc}(x) + c_{\Yc}(y)$ for some $c_{\Xc} \in L^1(\mu)$ and $c_{\Yc} \in L^1(\nu)$,
where $L^1(\mu)$ denotes a Lebesgue space with integral function with the measure $\mu$.
Then, there is a duality:
\begin{align}
\min_{\pi \in \Pi (\mu, \nu)} \int_{\Xc \times \Yc} c(x,y) d\pi(x,y)  &= \sup_{ \varphi \in L^1(\mu)} \Big\{ \int_{\Xc} \varphi(x) d\mu(x) + \int_{\Yc} \varphi^c (y) d\nu(y) \Big\} \label{eq:Kdual1} \\
&=\sup_{ \psi \in L^1(\mu)} \Big\{ \int_{\Xc} \psi^c(x) d\mu(x) + \int_{\Yc} \psi (y) d\nu(y) \Big\} \label{eq:Kdual2}
\end{align}
where $$\Pi(\mu, \nu) :=\{\pi~|~ \pi(A\times \Yc)=\mu(A),~~
 \pi(\Xc\times B)=\nu(B) \},$$
and the above maximum is taken over the so-called {\em Kantorovich potentials} $\varphi$ and $\psi$, whose c-transforms
are defined as
\begin{eqnarray}\label{eq:ctransform}
\varphi^c (y) := \inf_{x} (c(x,y) - \varphi(x) )&,\quad& \psi^c (x) := \inf_{y} (c(x,y) - \psi(y) )
\end{eqnarray}
\end{theorem}

In the Kantorovich dual formulation, finding the proper space of $\varphi$ and the computation of the $c$-transform $\varphi^c$ are important.
In particular, when 
 $c(x,y)=\|x-y\|$, we can reduce possible candidate of $\varphi$ to 1-Lipschitz functions so that it can simplify $\varphi^c$ to $-\varphi$ \cite{villani2008optimal}:
\[
\min_{\pi \in \Pi (\mu, \nu)} \int_{\Xc \times \Xc} ||x-y|| d\pi(x,y) = \sup_{ \varphi \in \text{Lip}_1(\Xc)} \Big\{ \int_{\Xc} \varphi(x) d\mu(x) - \int_{\Xc} \varphi (y) d\nu(y) \Big\},
\]
where Lip${}_1(\Xc) = \{\varphi \in L^1(\mu) : |\varphi(x)-\varphi(y)| \le ||x-y||\}$.

\subsection{PLS with Deep Learning Prior}
\label{sec:PLS}

Recently,  PLS frameworks  using  a deep learning prior have been
extensively studied \cite{zhang2017learningdenoiser,aggarwal2018modl}  thanks to their similarities to the classical
regularization theory. The main idea of these approaches is to utilize a pre-trained neural
network to stabilize the inverse solution.
For example, in model based deep learning architecture (MoDL) \cite{aggarwal2018modl}, the problem is formulated as
\begin{eqnarray}\label{eq:costPLS}
\min_xc(x;y,\Theta,\Hc)=\|y-\Hc x\|^2+ \lambda\|x- Q_\Theta(x) \|^2
\end{eqnarray}
for some regularization parameter $\lambda>0$, where $Q_\Theta(x)$ is a pre-trained denoising CNN with the network parameter $\Theta$ and the input $x$.
In \eqref{eq:costPLS}, the regularization term penalizes the difference between $x$  and the ``denoised'' version of $x$ so that
the regularization term gives high penalty  when $x$ is contaminated  with reconstruction artifacts.
An alternating minimization approach for \eqref{eq:costPLS} was proposed by the authors in \cite{aggarwal2018modl}, which can be stated
as follows:
\begin{eqnarray}
x_{n+1} = \arg\min_x \|y-\Hc x\|^2+ \lambda\|  x-z_n\|^2,&\quad \mbox{where}~z_{n} = Q_\Theta(x_{n}) . 
\label{eq:denoiser}
\end{eqnarray}

Another type of inversion approach using a deep learning prior is the so-called deep image prior (DIP) \cite{ulyanov2018deep}.
Rather than using an explicit prior,  
the deep neural network architecture itself is used as a regularization by restricting the solution space:
\begin{eqnarray}\label{eq:DIP}
\min_\Theta c(\Theta;y,\Hc)=\|y-\Hc Q_\Theta(z)\|^2
\end{eqnarray}
where $z$ is a random vector. Then,  the final solution becomes $x=Q_{\Theta^*}(z)$ with $\Theta^*$ being the estimated network parameters.

Yet other approaches are the generative models \cite{van2018compressed,bora2017compressed,wu2019deep} which use \eqref{eq:DIP} as the cost function.
However, they either estimate the random variable
$z$ from \eqref{eq:DIP} by fixing $\Theta$, or attempt to estimate both the random $z$ and the network weight $\Theta$.

\subsection{CycleGAN}
\label{sec:limitation}
CycleGAN was introduced in unmatched image-to-image translation task  in  computer vision \cite{zhu2017unpaired}.
The main idea is to impose a cycle-consistency  to mitigate the mode-collapsing
behavior in the generative adversarial network (GAN) \cite{goodfellow2014generative}.
Specifically, for the image translation problem between two domains $\Xc$ and $\Yc$,
 CycleGAN has two generators $G_\Theta:\Yc\mapsto\Xc$ and $F_\Upsilon:\Xc\mapsto \Yc$ 
and two discriminators $\varphi_\Phi$ and $\psi_\Xi$ 
 as shown in Fig.~\ref{fig:cycleGAN}(a),  which
are implemented using neural networks with the weight parameters $\Theta,\Upsilon,\Phi$ and $\Xi$, respectively \cite{zhu2017unpaired}.
Then, the unknown generator parameters $\Theta$ and $\Upsilon$ are estimated by solving the following min-max problem \cite{zhu2017unpaired}:
\begin{align}
\min_{\Theta,\Upsilon}\max_{\Phi,\Xi}\ell_{cycleGAN}(\Theta,\Upsilon;\Phi,\Xi) 
\end{align} 
where the loss is defined by
\begin{align}
\ell_{cycleGAN}(\Theta,\Upsilon;\Phi,\Xi) = \ell_{cycle}(\Theta,\Upsilon)+\ell_{Disc}(\Theta,\Upsilon;\Phi,\Xi)
\end{align} 
Here,  $\ell_{cycle}(\Theta,\Upsilon)$ is the cycle-consistency term given by
\begin{align}
\ell_{cycle}(\Theta,\Upsilon)  := \int_{\Xc} \|x- G_\Theta(F_\Upsilon(x)) \|  d\mu(x) +\int_{\Yc} \|y-F_\Upsilon(G_\Theta(y))\|   d\nu(y)
\end{align}
and 
$\ell_{Disc}(\Theta,\Upsilon;\Phi,\Xi)$ is a discriminator loss:
\begin{eqnarray}
\ell_{Disc}(\Theta,\Hc;\Phi,\Xi) 
&=& \int_{\Xc} \log (\varphi_\Phi(x))d\mu(x) + \int_\Yc \log(1-\varphi_\Phi(G_\Theta(y))) d\nu(y)  \\ 
&& + \int_{\Yc} \log(\psi_\Xi(y))d\nu(y) + \int_\Xc \log(1-\psi_\Xi(F_\Upsilon(x)))d\mu(x) \notag
\end{eqnarray}
where $\mu$ and $\nu$ are probability measures for $\Xc$ and $\Yc$, respectively.
The loss function is summarized in Table~\ref{tbl:algorithm}(a).
In cycleGAN,  $\varphi_\Phi$ is the discriminator that tries to differentiate true $x\in \Xc$ and the fake one generated by $G_\Theta(y)$, and
$\psi_\Xi$ is the discriminator that tells the true $y\in \Yc$ from the generated one by $F_\Upsilon(x)$.
Thanks to the competition between the discriminators and generators, more realistic images can be generated.
Moreover, the cycle-consistency loss imposes the one-to-one mapping condition between the two domains, so that the mode collapsing
problem in the conventional GAN \cite{goodfellow2014generative,zhu2017unpaired} can be reduced.

Although cycleGAN was originally proposed for style transfer in computer vision tasks,  
this framework was successfully employed for some inverse problems in medical imaging area such as low-dose CT \cite{kang2019cycle}.
However, blind application of cycleGAN for inverse problems is still problematic due to several issues.
First, the use of deep neural networks for two pairs of  generators and discriminators makes the cycleGAN training often difficult, since the four deep neural
networks should be trained simultaneously. For example, if one of them is not trained properly, all the other networks do not converge correctly.
Another  fundamental issue of using cycleGAN for inverse problems is that although the forward operator in \eqref{eq:g} is often
available in inverse problems, it is not clear how to incorporate this information during the cycleGAN training.
The main goal of this paper is  therefore to provide a systematic framework to address these issues.

\begin{figure}[!h]
	\center{
		\includegraphics[width=1\textwidth]{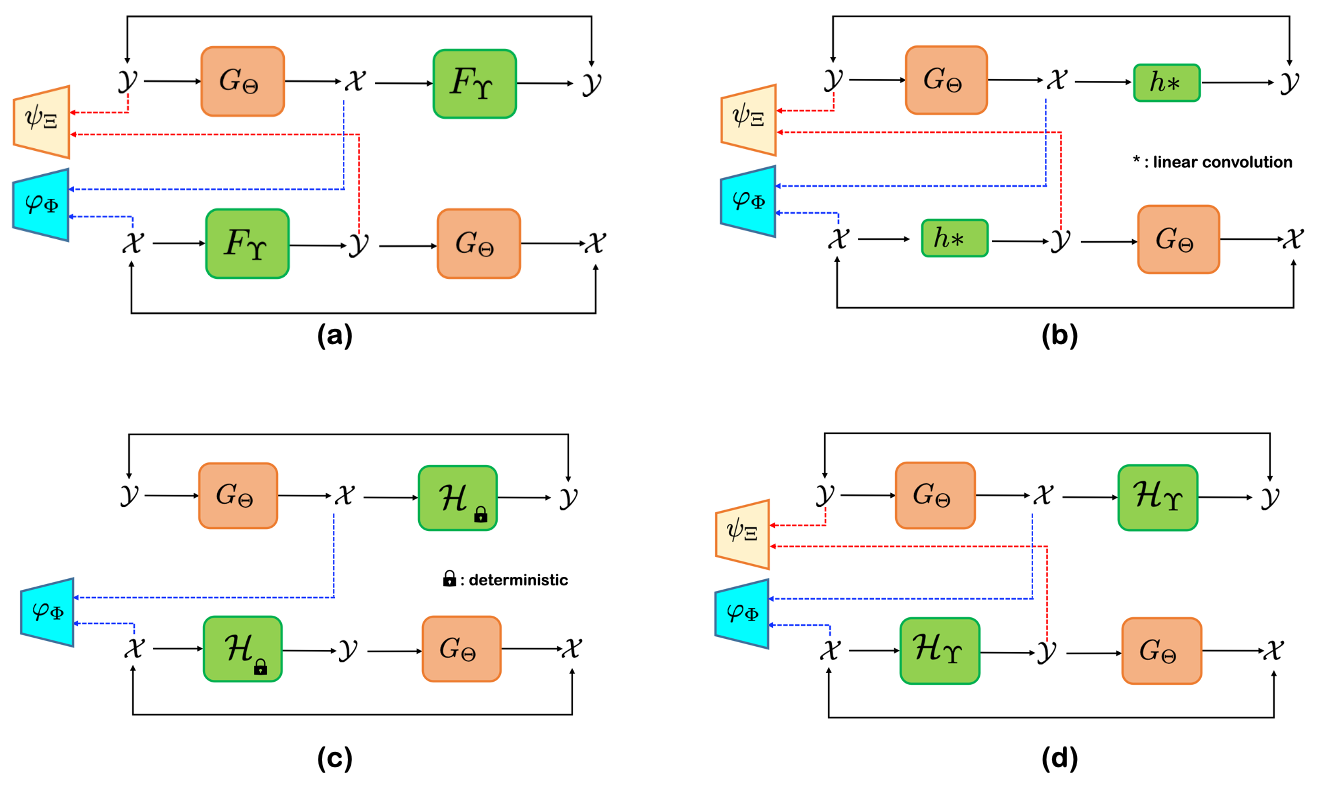}
	}\vspace{-0.3cm}
	\caption{Comparison of CycleGAN and OT-CycleGAN architectures.
		 (a) Conventional CycleGAN architecture. There are two generators ($G_\Theta$, $F_\Upsilon$) and two discriminators ($\varphi_\Phi$ and $\psi_\Xi$), which should be implemented using deep neural networks. (b) OT-CycleGAN where the forward physics can be described by linear convolution
		 operation $h\ast$.  In this case,  one deep neural network generator can be replaced by a linear layer.  
		 (c) OT-cycleGAN where the forward physics $\Hc$ is known.  In this case, only a single pair of generator $G_\Theta$ and discriminator  $\varphi_\Phi$ are needed.
		 (d) OT-cycleGAN where the forward physics is unknown. In this case, we need two pairs of generators  ($G_\Theta$, $F_\Upsilon$) and discriminators ($\varphi_\Phi$ and $\psi_\Xi$) implemented by deep neural networks, similar to the conventional cycleGAN.
		 }
	\label{fig:cycleGAN}
\end{figure}

\section{Optimal Transport Driven CycleGAN (OT-CycleGAN)}
\label{sec:main}

\subsection{Derivation}

To address the aforementioned limitation of cycleGAN for inverse problems,
here we explain our novel optimal transport driven cycleGAN (OT-cycleGAN) which incorporates
the forward physics  using optimal transport theory.
We show that our OT-cycleGAN is so general that
depending on the amount of prior knowledge of the forward mapping, various forms of cycleGAN architecture
can be derived as shown in Figs.~\ref{fig:cycleGAN}(b)-(d).
 In the following, we will show how the optimal transport theory
and  the forward
model can be exploited systematically in deriving a unified framework.

Specifically,  our OT-cycleGAN starts with a new PLS cost function with a novel deep learning prior as follows:
\begin{eqnarray}\label{eq:cost}
c(x,y;\Theta,\Hc)=\|y-\Hc x\|^q+ \lambda\| G_\Theta(y) - x\|^p 
\end{eqnarray}
with  $p,q\geq 1$,
$G_\Theta$ is a generative network with parameter $\Theta$ and $\Hc$ is a measurement system that should be estimated.
In \eqref{eq:cost}, $\lambda$ is the regularization term. 
For simplicity, in the rest of  the paper,
we assume $\lambda=1$.

Recall that the classical PLS approach \eqref{eq:problem} 
introduces a penalty term for choosing the solution $x$ based on the prior distribution of the data (see Fig.~\ref{fig:concept}(a)). 
This reduces the feasible sets, which may mitigate the ill-posedness of inverse problems.
On the other hand, the main motivation of the new PLS cost in \eqref{eq:cost} is to provide  another way to resolve the ambiguity of the feasible solutions. 
 Specifically, if we define a {single-valued} function $G_\Theta(y)$ and
impose the constraint $x=G_{\Theta^*}(y)$ with the learned parameter $\Theta^*$, many of the feasible solutions for $y=\Hc x$ can be pruned out as shown in Fig.~\ref{fig:concept}(b).
Therefore, this  provides another way of mitigating the ill-posedness.
It is also important to emphasize the difference between the regularization term in MoDL in \eqref{eq:costPLS}
and our regularization term in \eqref{eq:cost}.  Although they look somewhat similar,
there exists a fundamental differences. In MoDL, the  neural network $Q_\Theta(x)$ is a {\em pre-trained} denoising network which gets reconstruction
$x$ as input. Therefore, $x-Q_\Theta(x)$ can be assumed as noises and the regularization term penalizes the noise term.
On the other hand, our neural network  $G_\Theta$ gets the measurement $y$ as the input to estimate unknown $x$ as our neural network output.
Therefore, it penalizes the inconsistency in the inverse path.

Another important difference is that our neural network $G_\Theta$ is not pretrained, and the cost function in \eqref{eq:cost} is used to train the neural network $G_\Theta$.
Accordingly, if the PLS cost  becomes zero after the training,   we have
\begin{eqnarray}\label{eq:primal}
 y=\Hc x,\quad x=G_{\Theta^*}(y)
\end{eqnarray}
where $\Theta^*$ is the trained neural network weight.
Thus, $G_{\Theta^*}$ can be an inverse of the forward operator $\Hc$, which is the ultimate goal
in every inverse problem. In fact, this is the main motivation
of using the new penalty function in our formulation.
Therefore, the remaining question is how to estimate the parameter $\Theta^*$, where the optimal transport
plays an important role.

\begin{figure*}[b!]
\centering\includegraphics[width=0.7\textwidth]{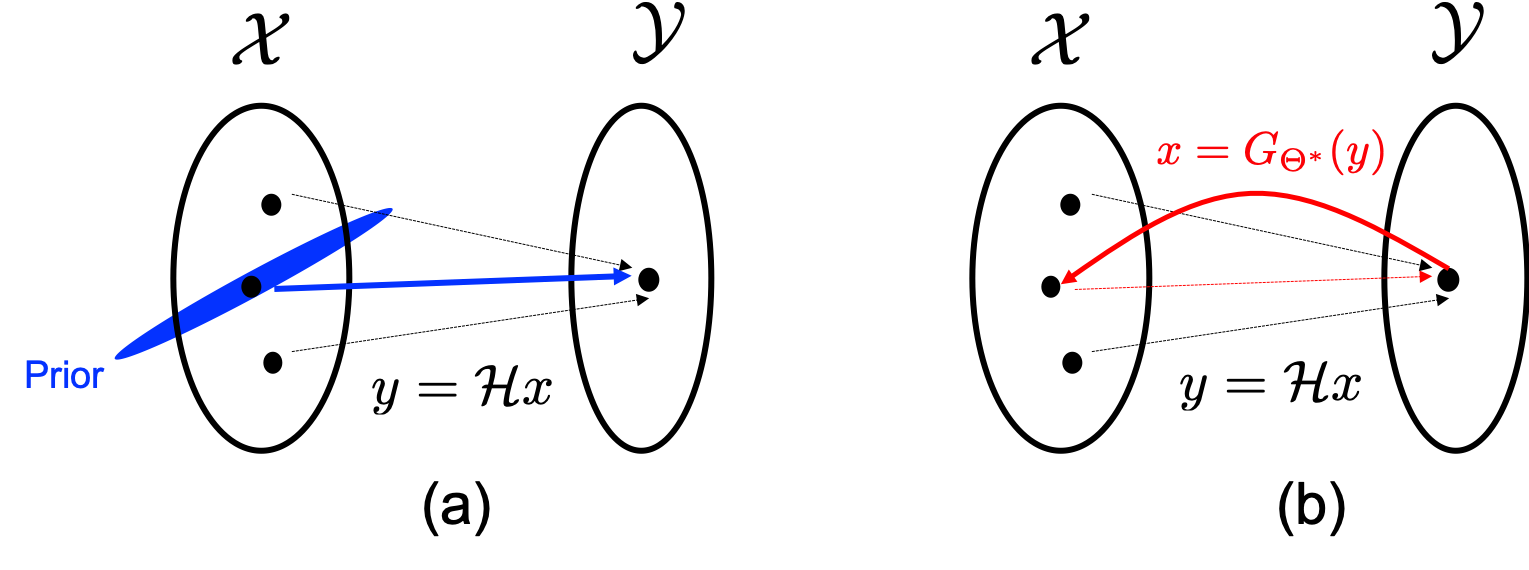}
\caption{Two strategies for resolving ambiguities in the feasible solutions in an ill-posed inverse problem. (a) Classical PLS approach using a close form prior distribution,
and (b) our PLS approach using an inverse mapping to define a prior. }
\label{fig:concept}
\end{figure*}

Specifically, our goal is to estimate the parameter $\Theta^*$ under the  unsupervised learning scenario where both $x$ and $y$ are unpaired.
Since we do not have exact match between $\Xc$ and $\Yc$ due to the unsupervised training, rather than attempting to find $x$ given $y$,
we estimate joint distribution, which can considers all combinations of $x\in \Xc$ and $y \in \Yc$.
In this scenario, $x$ and $y$ can be modeled as random samples from the probability distributions on $\Xc$ and $\Yc$, respectively,
but their joint probability measure $\pi(x,y)$, which we attend, is unknown.
In this regard, the problem is to find optimal probability measure which minimizes the average PLS cost among joint measures whose marginal distributions follow the samples $x's$ and $y's$ distributions.
This is equivalent to the optimal transport problem  \cite{villani2008optimal,peyre2019computational} which minimizes
 the average transport cost, where the average
cost can be computed by
\begin{eqnarray}\label{eq:org}
\Kd(\Theta,\Hc):= \min_{\pi} \int_{\Xc\times \Yc} c(x,y;\Theta,\Hc)d\pi(x, y) 
\end{eqnarray}
where the minimum is taken over all joint distributions whose marginal distributions with respect to $X$ and $Y$ are $\mu$ and $\nu$, respectively.
Then, the unknown parameters for inverse problem, for example, $\Hc$ and $\Theta$ for the forward and inverse operator, respectively,
can be found by minimizing $\Kd(\Theta,\Hc)$ with respect to these parameters.

Eq.~\eqref{eq:org} is the {\em primal} form of the optimal transport problem. 
However, the primal formulation is complicated to solve directly due to several technical issues.
For example,
	suppose $\Kd(\Theta,\Hc)$ is achieved when the optimal  joint distribution is $\pi^*$.
	Now, if we attempt to minimize $\Kd$ with respect to $\Theta^*$, then the cost also varies, so the optimal $\pi$ should be modified.
	Therefore, we should seek alternating optimization of $\pi$ and $\Theta$. 
		Unfortunately, since $\pi$ is a distribution on a high dimensional space, 
		nonparametric modeling of the distribution for such alternating optimization
		is a very difficult problem.

To address this problem, our goal is to obtain its {\em dual} formulation
using the Kantorovich dual formula, where the estimation of the joint distribution is not necessary.
Although the additive form of the PLS cost in \eqref{eq:cost} makes it difficult
to obtain the explicit dual formula,
 Proposition~\ref{prp:bound} shows that there exists interesting bounds that can be used to obtain the dual formulation.

\begin{proposition}\label{prp:bound}
Consider the  transportation cost $\Kd(\Theta,\Hc)$ of the primal OT problem using the PLS cost \eqref{eq:cost} with  $p=q=1$:
\begin{align}\label{eq:Kd}
\Kd(\Theta,\Hc)=\min_\pi \int_{\Xc\times \Yc} \|y-\Hc x\|+\| G_\Theta(y) - x\|d\pi(x,y) \ .
\end{align}
Now let us define
\begin{eqnarray}\label{eq:Dd}
\Dd(\Theta,\Hc):=  \frac{1}{2} \ell_{cycle}(\Theta,\Hc) +\ell_{OT'}(\Theta,\Hc) 
\end{eqnarray}
where
\begin{align}
\ell_{cycle}(\Theta,\Hc)  =& \frac{1}{2}\left\{\int_{\Xc} \|x- G_\Theta(\Hc x) \|  d\mu(x) +\int_{\Yc} \|y-\Hc G_\Theta(y)\|   d\nu(y)\right\} \label{eq:cycle}
\end{align}
\begin{align}\label{eq:OT'}
\ell_{OT'}(\Theta,\Hc) =& 
 \frac{1}{2}\left\{\max_{\varphi}\int_\Xc \varphi(x)  d\mu(x) - \int_\Yc \varphi(G_\Theta(y))d\nu(y) \right.  \\
 & \left. + \max_{\psi}\int_{\Yc} \psi(y)  d\nu(y) - \int_\Xc \psi(\Hc x)  d\mu(x)\right\}  \notag
\end{align}
where $\varphi,\psi$ are 1-Lipschitz functions.
Then,  the average transportation cost $\Kd(\Theta,\Hc)$  in \eqref{eq:Kd}  can be approximated
by $\Dd(\Theta,\Hc)$ with the following error bound:
\begin{eqnarray}
|\Kd(\Theta,\Hc)-\Dd(\Theta,\Hc)| &\leq \frac{1}{2}\ell_{cycle}(\Theta,\Hc).
\end{eqnarray}
Moreover, the error bound $\frac{1}{2} \ell_{cycle}(\Theta,\Hc)$ becomes zero
	when $G_\Theta$ is an exact inverse of $\Hc$, i.e.
	$\Hc G_{\Theta}(y)=y$ and $G_{\Theta}\Hc x = x$ for all $x\in \Xc,y\in \Yc$.
\end{proposition}

\begin{proof}
We first define the optimal joint measure $\pi^*$ for the primal problem \eqref{eq:org} with $p=q=1$:
$$\Kd(\Theta,\Hc)=\min_\pi \int_{\Xc\times \Yc} c(x,y;\Theta,\Hc)\pi(dx, dy) =\int_{\Xc\times \Yc} c(x,y;\Theta,\Hc)\pi^*(dx, dy)$$
where
$ c(x,y;\Theta,\Hc):=\|y-\Hc x\|+\| G_\Theta(y) - x\|.$
Using the two forms of the Kantorovich dual formulations in \eqref{eq:Kdual1} and \eqref{eq:Kdual2}, we have
\begin{align*}
\Kd(\Theta,\Hc)=& \int_{\Xc\times \Yc}  \|y-\Hc x\|+\| G_\Theta(y) - x\| d\pi^*(x,y) \\
=\frac{1}{2}& \left\{\max_{\varphi}\int_\Xc\varphi(x)d\mu(x) + \int_\Yc \varphi^c(y) d\nu(y)+ \max_{\psi}\int_\Xc \psi^c(x) d\mu(x) + \int_\Yc \psi
(y)d\nu(y)\right\}\\
= \frac{1}{2}& \left\{\max_{\varphi}\int_\Xc \varphi(x) d\mu(x) 
+ \int_\Yc \inf_x \{ \|y-\Hc x\|+\| G_\Theta(y) - x\| -\varphi(x) \}d\nu(y)\right.\\
&\left. +\max_{\psi}\int_\Xc \inf_y \{ \|y-\Hc x\|+\| G_\Theta(y) - x\| -\psi(y) \}d\mu(x) + \int_\Yc \psi
(y)d\nu(y)\right\}
\end{align*}
Here, the need to find the optimal joint distribution $\pi^*(x,y)$ is replaced by the maximization with respect to the so-called
Kantorovich potentials $\varphi$ and $\psi$.
Now, instead of finding the $\inf_x$, we choose $x=G_\Theta(y)$; similarly, instead of $\inf_y$,  we chose $y=\Hc x$.
This leads to an upper bound: 
\begin{eqnarray*}
\Kd(\Theta,\Hc)&\leq & \frac{1}{2} \left\{\max_{\varphi}\int_\Xc \varphi(x) d\mu(x)-\int_\Yc \varphi(G_\Theta(y))d\nu(y) +\int_\Yc  \|y-\Hc G_\Theta(y)\| d\nu(y)\right. \\
&+ &\left. \max_{\psi}\int_{\Yc} \psi(y)  d\nu(y) - \int_\Xc \psi(\Hc x)  d\mu(x) +
\int_\Xc\| G_\Theta(\Hc x) - x\| d\mu(x)
\right\}\\
&=& \frac{1}{2}\left\{\max_{\varphi}\int_\Xc \varphi(x)  d\mu(x) - \int_\Yc \varphi(G_\Theta(y))d\nu(y)+ \max_{\psi}\int_{\Yc} \psi(y)  d\nu(y) - \int_\Xc \psi(\Hc x)  d\mu(x) \right. \\
&&\left.+\int_\Xc\| G_\Theta(\Hc x) - x\| d\mu(x)+ \int_\Yc \|y-\Hc G_\Theta(y)\|d\nu(y)\right\} \\
&=& \ell_{OT'}(\Theta,\Hc)+\ell_{cycle}(\Theta,\Hc)
\end{eqnarray*}
Now, using  1-Lipschitz continuity of the Kantorovich potentials, we have
\begin{eqnarray}
-\varphi(G_\Theta(y))\leq& \| G_\Theta(y) - x\| -\varphi(x) &\leq  \|y-\Hc x\|+\| G_\Theta(y) - x\|-\varphi(x) \label{eq:lower1}\\
-\psi(\Hc x)  \leq&  \|y-\Hc x\| -\psi(y) &\leq  \|y-\Hc x\|+\| G_\Theta(y) - x\| -\psi(y) \label{eq:lower2}
\end{eqnarray}
This leads to the following lower-bound
\begin{eqnarray*}
\Kd(\Theta,\Hc)&\geq&  \frac{1}{2}\left\{\max_{\varphi}\int_\Xc \varphi(x)  d\mu(x) - \int_\Yc \varphi(G_\Theta(y))d\nu(y)+ \max_{\psi}\int_{\Yc} \psi(y)  d\nu(y) - \int_\Xc \psi(\Hc x)  d\mu(x) \right\}  \\
&=& \ell_{OT'}(\Theta,\Hc)
\end{eqnarray*}
By collecting the two bounds, we have
$$\ell_{OT'}(\Theta,\Hc) \leq \Kd(\Theta,\Hc) \leq \ell_{OT'}(\Theta,\Hc)+\ell_{cycle}(\Theta,\Hc).$$
which leads to
\begin{eqnarray*}
| \Kd(\Theta,\Hc) -\Dd(\Theta,\Hc) | \leq \frac{1}{2}\ell_{cycle}(\Theta,\Hc).
\end{eqnarray*}
where $\Dd(\Theta,\Hc)$ is defined in \eqref{eq:Dd}.
Finally, the approximation error becomes zero when  $\ell_{cycle}(\Theta,\Hc)=0$. This happens
for  $(\Theta,\Hc)$ such that
$\Hc G_{\Theta}(y)=y$ and $G_{\Theta}\Hc x = x$ for all $x\in \Xc,y\in \Yc$, which is equal to say
that $G_{\Theta}$ is an exact inverse of $\Hc$.
\end{proof}

Recall that the optimal transportation map can be found by minimizing the
primal OT cost function $\Kd(\Theta,\Hc)$ with respect to $(\Theta,\Hc)$:
\begin{align}
\min\limits_{\Theta,\Hc}\Kd(\Theta,\Hc) 
\end{align}
Using Proposition~\ref{prp:bound}, this problem can be solved by a constrained optimization problem with sufficiently
small tolerance $\epsilon$ for $\ell_{cycle}$:
\begin{align} \label{eq:dualc}
\min\limits_{\Theta,\Hc}~&\Dd(\Theta,\Hc) \quad \mbox{subject to}~\ell_{cycle}(\Theta,\Hc) \le \epsilon
\end{align}
which can be solved using an unconstrained form using the
Lagrangian dual:
\begin{equation}
\inf_{\Theta, \Hc} ~L(\Theta, \Hc, \alpha) := ~\Dd(\Theta,\Hc) + \alpha (\ell_{cycle}(\Theta,\Hc)-\epsilon)
\end{equation}
for some $\epsilon$ dependent Lagrangian multiplier parameter $\alpha$.
Therefore,  the constrained optimization problem \eqref{eq:dualc} can be solved by an easier
 unconstrained optimization problem: 
\begin{align} \label{eq:dualu}
\min\limits_{\Theta,\Hc}~\Dd(\Theta,\Hc) +\alpha \ell_{cycle}(\Theta,\Hc)&\quad\mbox{or}&    \min\limits_{\Theta,\Hc}~\ell_{OT'}(\Theta,\Hc)+\gamma\ell_{cycle}(\Theta,\Hc)
\end{align}
where $\alpha$ or $\gamma:=\alpha+\frac{1}{2}$ is a hyperparameter to adjust during experiments.

Finally, by implementing the Kantorovich potential using a neural network with parameters $\Phi$ and $\Xi$,
i.e. $\varphi:=\varphi_\Phi$ and $\psi:=\psi_\Xi$, 
we have the following cycleGAN problem:
\begin{align}\label{eq:cycleGAN}
\min\limits_{\Theta,\Hc}  \ell_{OT'}(\Theta,\Hc) +\gamma\ell_{cycle}(\Theta,\Hc) =\min\limits_{\Theta,\Hc}\max\limits_{\Phi,\Xi}\ell(\Theta,\Hc;\Phi,\Xi)
\end{align}
where
\begin{eqnarray}
\ell(\Theta,\Hc;\Phi,\Xi) =  \ell_{OTDisc}(\Theta,\Hc;\Phi,\Xi) + \gamma\ell_{cycle}(\Theta,\Hc)\notag
\end{eqnarray}
where $ \ell_{cycle}(\Theta,\Hc)$ denotes the cycle-consistency loss in \eqref{eq:cycle}  and
$\ell_{OTDisc}(\Theta,h;\Phi,\Xi)$ is the discriminator loss given by:
\begin{eqnarray}\label{eq:ellGAN}
\ell_{OTDisc}(\Theta,\Hc;\Phi,\Xi) 
&=&\frac{1}{2} \left\{ \int_{\Xc} \varphi_\Phi(x)d\mu(x) - \int_\Yc \varphi_\Phi(G_\Theta(y)) d\nu(y)  \right. \\ 
&&\left.  + \int_{\Yc} \psi_\Xi(y)d\nu(y) - \int_\Xc \psi_\Xi(\Hc x)d\mu(x) \right\}\notag
\end{eqnarray}
Note that
  the  Kantorovich dual problems should be maximized with respect to all 1-Lipschitz functions, so
the parameterized functions $\varphi$ and $\psi$  should have sufficiently large capacity to approximate any 1-Lipschitz function. This is why we employ
deep neural network to model the 1-Lipschitz Kantorovich potentials.
In neural network implementation, the Kantorovich 1-Lipschitz potentials  $\varphi:=\varphi_\Phi$ and $\psi:=\psi_\Xi$
correspond to the W-GAN discriminators.  
Specifically, 
$\varphi_\Phi$ tries to find the difference between the true image $x$ and the generated image $G_\Theta(y)$,
whereas $\psi:=\psi_\Xi$ attempts to find the fake measurement data  that are generated by the synthetic
measurement procedure $\Hc x$.

\subsection{Special Cases of OT-cycleGAN}
\label{sec:cases}
One of the most important advantages of OT-cycleGAN is that 
depending on the amount of prior knowledge of the forward mapping, it can lead to various cycleGAN architectures:  for example,  the one
in Fig.~\ref{fig:cycleGAN}(d), which is basically equivalent to
the original cycleGAN architecture in Fig.~\ref{fig:cycleGAN}(a), 
or a novel architecture in Fig.~\ref{fig:cycleGAN}(b) where a linear generator can replace a complicated neural network 
generator, or a much simpler novel architecture in Fig.~\ref{fig:cycleGAN}(c), where only a single pair of generator
and discriminator are necessary.
 
 More specifically, suppose that the forward mapping is given by
\begin{eqnarray}
y &=&h\ast x + w \ ,
\end{eqnarray}
where $h$ is an unknown blur kernel and $w$ is noise.
For this, OT-cycleGAN framework suggests the following cost function for the PLS formulation:
\begin{eqnarray}\label{eq:forwardLin}
c(x,y;\Theta,h)=\|y-h\ast x\|+ \| G_\Theta(y) - x\| ,
\end{eqnarray}
which leads to the following Kantorovich dual formulation:
\begin{align}\label{eq:OTopt}
\min_{\Theta,\Upsilon}\max_{\Phi,\Xi}\ell_{OTcycleGAN}(\Theta,\Upsilon;\Phi,\Xi) 
\end{align} 
where the loss is defined by
\begin{align}
\ell_{OTcycleGAN}(\Theta,\Upsilon;\Phi,\Xi) = \gamma \ell_{cycle}(\Theta,\Upsilon)+\ell_{OTDisc}(\Theta,\Upsilon;\Phi,\Xi)
\end{align} 
The specific details of cycle-consistency loss $\ell_{cycle}$ 
and the discriminator loss $\ell_{OTDisc}$ are summarized
in Table~\ref{tbl:algorithm}(b). As compared to the standard cycleGAN in Table~\ref{tbl:algorithm}(a), which requires two deep neural network
based generators $G_\Theta$ and $F_\Upsilon$,  our OT-cycleGAN formulation replaces the $F_\Upsilon$ by the linear convolution
kernel $h$ as shown in Fig.~\ref{fig:cycleGAN}(b),  whose estimation is much simpler. Therefore, the overall neural network training becomes more stable.

In another example, let the forward model be given by \eqref{eq:g}, where the forward operator $\Hc$ is known.
Then, the corresponding PLS cost is given by
\begin{eqnarray}\label{eq:forwardKnown}
c(x,y;\Theta)=\|y-\Hc x\|+ \| G_\Theta(y) - x\| ,
\end{eqnarray}
where the unknown parameter is only $\Theta$.
The corresponding OT-cycleGAN has the similar optimization problem as \eqref{eq:OTopt} with a slight but important modification.
Specifically, as
the forward path has no uncertainty, the maximization of the discriminator $\psi_\Xi$  in \eqref{eq:cycleGAN} with respect to $\Xi$ does
not affect the generator $G_\Theta$. Therefore, the discriminator  $\psi_\Xi$ can be neglected.
This leads to the simpler cycle-consistency and discriminator loss as shown in Table~\ref{tbl:algorithm}(c).
The corresponding network architecture is shown in Fig.~\ref{fig:cycleGAN}(c), where only a single pair of generator and discriminator are
necessary.

Finally, suppose that the prior knowledge of the forward operator $\Hc$ in \eqref{eq:g} is not available or difficult to model.
In this case, we should also estimate the forward operator  using a neural network parameterized by $\Upsilon$.
Then, the corresponding PLS cost is given by
\begin{eqnarray}\label{eq:forwardUnknown}
c(x,y;\Theta,\Upsilon)=\|y-\Hc_\Upsilon(x)\|+ \| G_\Theta(y) - x\| ,
\end{eqnarray}
where the unknown parameters are  $\Theta$ and $\Upsilon$.
This leads to cycle-consistency and discriminator loss as shown in Table~\ref{tbl:algorithm}(d), which is basically
similar to the standard cycleGAN in Table~\ref{tbl:algorithm}(a) except for the specific discriminator term.
Accordingly, the corresponding network architecture is   Fig.~\ref{fig:cycleGAN}(d),  which is the same as the
standard cycleGAN in Fig.~\ref{fig:cycleGAN}(a), where two pairs of generators and discriminators should be estimated.

\begin{table}[!thb]
\centering
\caption{Comparison of standard cycleGAN and OT-cycleGAN in various forward models. }
\label{tbl:algorithm}
\resizebox{0.99\textwidth}{!}{
	\begin{tabular}{cc|c|c}
		\hline
			& {Algorithm}					& {$\ell_{cycle}$}			& $\ell_{Disc}$ or $\ell_{OTDisc}$				\\ \hline\hline
	\multirow{2}{*}{(a)}  &  		\multirow{2}{*}{CycleGAN}   					& $\int_{\Xc} \|x- G_\Theta(F_\Upsilon(x)) \|  d\mu(x) $		&$ \int_{\Xc} \log (\varphi_\Phi(x))d\mu(x) + \int_\Yc \log(1-\varphi_\Phi(G_\Theta(y))) d\nu(y)$		 \\ 
	&		& $+\int_{\Yc} \|y-F_\Upsilon(G_\Theta(y))\|   d\nu(y)$		&$ + \int_{\Yc} \log(\psi_\Xi(y))d\nu(y) + \int_\Xc \log(1-\psi_\Xi(F_\Upsilon(x)))d\mu(x)$		 \\ \hline
	\multirow{2}{*}{(b)}  &  \multirow{2}{*}{OT-CycleGAN for \eqref{eq:forwardLin}}  						& $\int_{\Xc} \|x- G_\Theta(h\ast x) \|  d\mu(x) $		&$ \int_{\Xc} \varphi_\Phi(x)d\mu(x) - \int_\Yc \varphi_\Phi(G_\Theta(y)) d\nu(y)$		 \\ 
&			& $+\int_{\Yc} \|y-h \ast G_\Theta(y) \|   d\nu(y)$		&$ + \int_{\Yc} \psi_\Xi(y)d\nu(y) - \int_\Xc \psi_\Xi(h\ast x)d\mu(x)$		 \\ \hline
	\multirow{2}{*}{(c)}  &  \multirow{2}{*}{OT-CycleGAN for \eqref{eq:forwardKnown}}  						& $\int_{\Xc} \|x- G_\Theta(\Hc x) \|  d\mu(x) $		&$ \int_{\Xc} \varphi_\Phi(x)d\mu(x) - \int_\Yc \varphi_\Phi(G_\Theta(y)) d\nu(y)$		 \\ 
&			& $+\int_{\Yc} \|y-\Hc G_\Theta(y) \|   d\nu(y)$		&		 \\ \hline
	\multirow{2}{*}{(d)}  &  		\multirow{2}{*}{OT-CycleGAN for \eqref{eq:forwardUnknown}}   					& $\int_{\Xc} \|x- G_\Theta(\Hc_\Upsilon(x)) \|  d\mu(x) $		&$ \int_{\Xc} \varphi_\Phi(x)d\mu(x) - \int_\Yc \varphi_\Phi(G_\Theta(y)) d\nu(y)$		 \\ 
	&		& $+\int_{\Yc} \|y-\Hc_\Upsilon(G_\Theta(y))\|   d\nu(y)$		&$ + \int_{\Yc} \psi_\Xi(y)d\nu(y) - \int_\Xc \psi_\Xi(\Hc_\Upsilon(x))d\mu(x)$		 \\ \hline
	\end{tabular}
		}
\end{table}

 Therefore, we can see that in our OT-cycleGAN formulation,
  the physics-driven data consistency term
can simplify the neural network architecture, and also provide a constraint such that the trained neural network can 
  generate physically meaningful output.

\subsection{Comparison with Other Generalization}

Another important advantage of the proposed method is that the unknown image $x$ can be obtained as an output of the
{\em feedforward} neural network $G_\Theta(y)$ for the given measurement $y$. This makes the inversion procedure much simpler.

One may say that the PLS cost function for the deep image prior (DIP) in \eqref{eq:DIP}
can produce a different variation of unsupervised learning. However, the resulting
formulation is not a feed-forward neural network.
Specifically, the primal problem of the OT problem with respect to the DIP cost in \eqref{eq:DIP}
is given by
\begin{eqnarray}
\min_{\pi\in \Pi(\mu, \eta)}\int_{\Xc\times \Zc} \| y - \Hc G_\Theta(z)\| d\pi(x, z) 
\end{eqnarray}
where we consider $y$ and $z$ as random variable, $\mu$,$\nu$ and $\eta$ represent distributions of $x$,$y$ and $z$ respectively, and we use non-squared norm for simplicity.
This leads to the following dual formulation:
\begin{align*}
\min_\Theta\min_\pi\int_{\Xc\times \Zc} \| y - \Hc G_\Theta(z)\| d\pi(x, z) 
=  \min_\Theta\max_{\psi}\int_{\Yc} \psi(y) d\nu(y) - \int_\Zc \psi(\Hc G_\Theta(z))  d\eta(z) .
\end{align*}
where $\psi$ is a 1-Lipschitz function.
Although this is a nice way of pretraining a deep image prior model using an unmatched training data set,
the final image estimate still comes from the following optimization problem:
$$x = G_\Theta(z^*) \quad \mbox{where}\quad z^* = \arg\min_z \|y-\Hc G_{\Theta^*}(z)\|$$
where $\Theta^*$ is the estimated network parameters from previous training step.
This is equivalent to the deep generator model \cite{bora2017compressed}, which
is not a feed-forward neural network and requires additional optimization at the test phase.

\subsection{Implementation}
  
Note that we only consider $p=q=1$ due to the inequalities \eqref{eq:lower1} and \eqref{eq:lower2} for 1-Lipschitz $\varphi$ and $\psi$.
However, the use of the general PLS  cost  would be interesting, and it may lead to an interesting
 variation of the cycleGAN architecture. This could be  done  using a regularized version of OT \cite{peyre2019computational}.

Having said this, imposing 1-Lipschitz condition for the discriminator is the main idea of the Kantorovich dual formulation as in W-GAN \cite{arjovsky2017wasserstein}; therefore, 
care should be taken  to ensure that the Kantorovich potential  becomes 1-Lipschitz.
There are many approaches to address this. For example, in the original W-GAN paper \cite{arjovsky2017wasserstein}, the weight
clipping was used to impose 1-Lipschitz condition.
Another method is to use the spectral normalization method  \cite{miyato2018spectral}, which utilizes the power iteration method to impose
constraint on the largest singular value of weight matrix in each layer.
Yet another popular method is the  WGAN with the gradient penalty (WGAN-GP), where
the gradient of the Kantorovich potential is constrained to be 1 \cite{gulrajani2017improved}.
Specifically,  for the case of  Table~\ref{tbl:algorithm}(c),
$\ell_{OTDisc}$ is modified as
\begin{align}\label{eq:ourWGANnonblind2}
&\ell_{OTDisc}(\Theta;\Phi) \\
&= \left(\int_\Xc \varphi_\Phi(x)d\mu(x) - \int_\Yc \varphi_\Phi(G_\Theta(y))d\nu(y)\right) \notag \\
&-\eta \int_{{\Xc}}(\|\nabla_{\tilde x}\varphi_\Phi(x)\|_2 - 1)^2d\mu(x) \notag
\end{align}
where 
$\eta>0$ is the regularization parameters to impose 1-Lipschitz property for the discriminators, and 
  $\tilde x=\alpha x+(1 - \alpha)G_\Theta(y)$ with $\alpha$ 
being random variables from the uniform distribution between $[0,1]$ \cite{gulrajani2017improved}. 
In this paper,  we use WGAN-GP as our implementation to impose 1-Lipschitz constraint.

Finally, the pseudocode implementation of our OT-cycleGAN is shown in Algorithm~\ref{alg:Pseudocode} for the
case of Table~\ref{tbl:algorithm}(d). Algorithms for other variation of OT-cycleGAN can be simply modified from Algorithm~\ref{alg:Pseudocode}.

\begin{algorithm}
	\caption{Pseudocode implemenation of OT-CycleGAN in Table~\ref{tbl:algorithm}(d).}
	\label{alg:Pseudocode}
	\begin{algorithmic}[1]
		\State Given: unpaired training samples $\{x^{(n)}\}_{n=1}^N$ and $\{y^{(m)}\}_{m=1}^M$ from $(X,\mu)$ and $(Y,\nu)$.
		\State Initialization: 
		 $\Theta:=\Theta^{(0)}, \Upsilon:=\Upsilon^{(0)},\Phi:=\Phi^{(0)}, \Xi:=\Xi^{(0)}.$
		\For{$k=0:K$}
		\For{each mini-batch}
		\State 
		$(\Phi^{(k)},\Xi^{(k)}):= \arg\max\limits_{\Phi,\Xi}\ell_{OTDisc}(\Theta^{(k)},\Upsilon^{(k)};\Phi,\Xi)$
		\EndFor
		\For{each mini-batch}
		\State 
		$(\Theta^{(k)},\Upsilon^{(k)}):= \arg\min\limits_{\Theta,\Upsilon}\gamma \ell_{cycle}(\Theta,\Upsilon)+\ell_{OTDisc}(\Theta,\Upsilon;\Phi^{(k)},\Xi^{(k)})$
		\EndFor
		\EndFor
	\end{algorithmic}
\end{algorithm}

\section{Experimental Results}
\label{sec:exp}
Among many potential  applications, here we  discuss its application to accelerated MRI, super-resolution
microscopy, and low-dose CT problems. 
Before we dive into three applications, we briefly compare the physics of each case.
In Section \ref{sec:mri}, the forward operator is subsampling in Fourier domain, and we know the pattern of subsampling.
In other word, we have full knowledge about the forward operator.
In Section \ref{sec:microscopy},  the measurement
comes from  is linear blurring kernel, even though the exact value of the kernel is unknown.
In the last example of Section \ref{sec:ct},  the goal is to remove noises from  low-dose CT image,
and due to the complicated CT forward physics at low-dose acquisition, the forward operator is not known.
Therefore, these three examples corresponds to the special cases of OT-cycleGAN in Table~\ref{tbl:algorithm}(b)-(d),
and the goal of the experiments is to demonstrate that OT-cycleGAN is general enough to cover all these cases.

\subsection{Accelerated MRI}
\label{sec:mri}
In accelerated magnetic resonance imaging (MRI), the goal is to recover high-quality MR images from sparsely sampled $k$-space data to reduce the acquisition time.  
This problem has been extensively studied  using compressed sensing \cite{lustig2007sparse}, but recently, deep learning approaches have been the
main research interest due to their excellent performance and significantly reduced run-time complexity \cite{hammernik2018learning,han2018k}.

A standard deep learning method for accelerated MRI is based on supervised learning, where the MR images from fully sampled $k$-space
data are used as references and  subsampled $k$-space data are used as the input for the neural network for training.
Unfortunately, in accelerated MRI, the acquisition of
high-resolution fully sampled $k$-space data takes significant long time, and often requires changes of  the standard acquisition protocols.
Therefore, collecting sufficient amount of training data is a major huddle in practice,
and the need for unsupervised learning without matched reference data is increasing.

\begin{figure}[!hbt]
  \center{
   \includegraphics[width=0.7\textwidth]{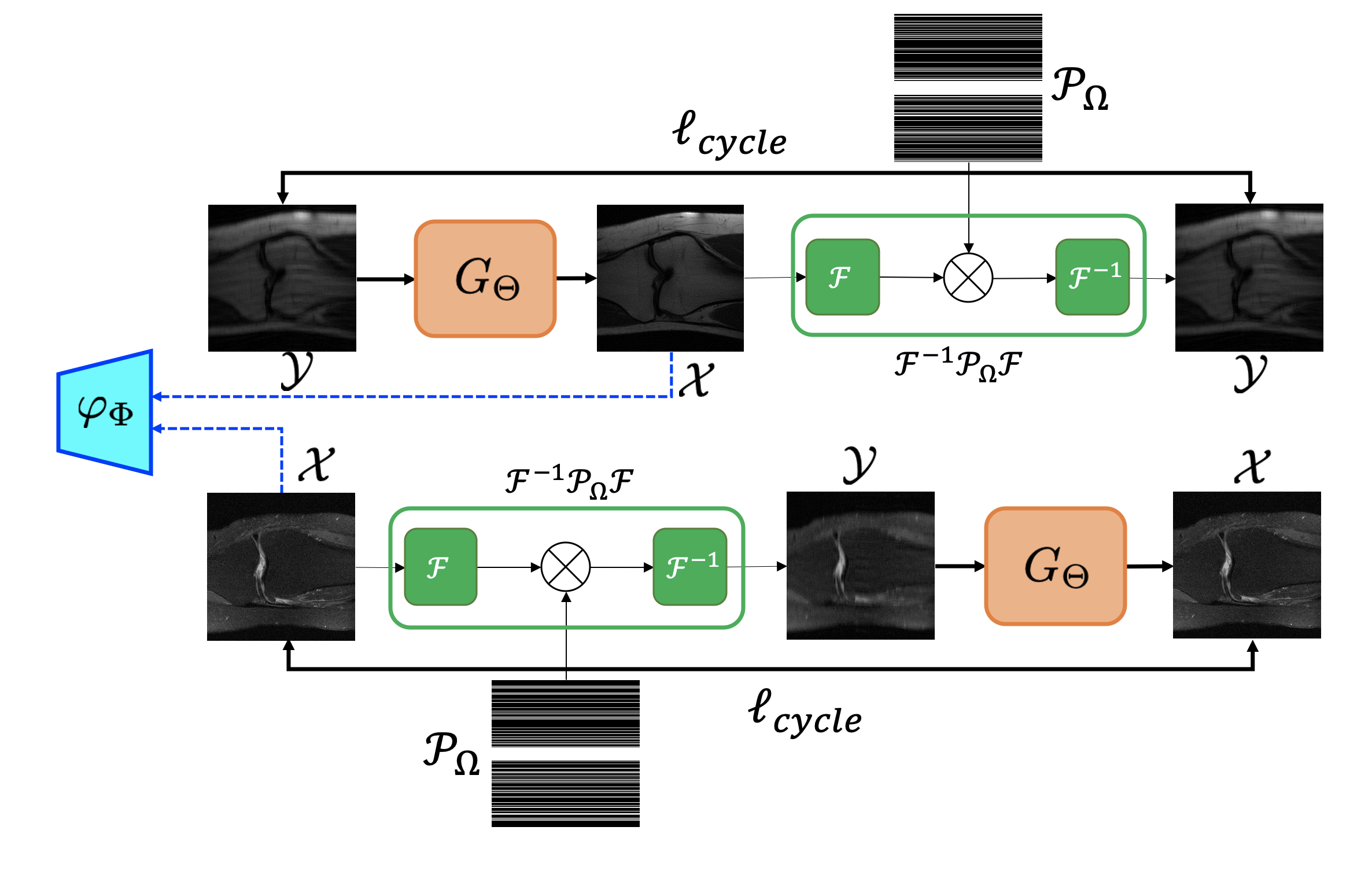}
}
  \caption{Proposed cycleGAN architecture with 1-D downsampling patterns for accelerated MRI.
  		This corresponds to OT-CycleGAN in Fig.~\ref{fig:cycleGAN}(c) since the forward operation is a deterministic subsampled Fourier transform. Therefore, we only need one discriminator  and one generator.
  }
  \label{fig:mr_arch}
\end{figure}

In accelerated MRI, the forward measurement model can be described as
\begin{eqnarray}\label{eq:forwardMR}
\hat x=  \Pc_\Omega \Fc x+w
\end{eqnarray}
where $\Fc$ is the 2-D Fourier transform, $w$ is the measurement error, and $\Pc_\Omega$ is the projection to  $\Omega$ that denotes $k$-space sampling indices.
To implement every step of the algorithm as an image domain processing,  \eqref{eq:forwardMR} can be converted to the image domain forward model
by applying the inverse Fourier transform:
\begin{eqnarray}
y=  \Fc^{-1} \Pc_\Omega \Fc x+\Fc^{-1} w
\end{eqnarray}
This results in the following  cost function for the PLS formulation:
\begin{eqnarray}
c(x,y;\Theta)=\|y-\Fc^{-1} \Pc_\Omega \Fc x\|+ \| G_\Theta(y) - x\| 
\end{eqnarray}

One of the important considerations in designing the OT-cycleGAN architecture for accelerated MRI is that
the sampling mask $\Omega$ is known so that the forward mapping for the inverse problem is fully known.
This implies that the forward path has no uncertainty and we do not need to estimate $\Hc= \Fc^{-1} \Pc_\Omega \Fc$.
This corresponds to the OT-cycleGAN in Table~\ref{tbl:algorithm}(c).

The schematic diagram of the resulting cycleGAN architecture is illustrated in Fig.~\ref{fig:mr_arch},
where only a single generator and discriminator pairs is necessary. 
The leftmost images are  image samples of $\Xc$ and $\Yc$, which are from full and under-sampled MR $k$-space data.
Given the samples, the generator network and the forward operator generate  fake fully sampled image and 
an artifact-corrupted image, and then return them back to original ones. 
	Among this, the fake fully sampled image is passed to a discriminator. 
Note that we just need a single generator and discriminator, as discussed before.

We use single coil dataset from fastMRI challenge \cite{zbontar2018fastmri} for our experiments. This dataset is composed of MR images of knees. We extracted 3500 MR images from fastMRI single coil validation set.
The dataset is paired, but we do not utilize paired information during training. The paired information is only used to evaluate metric during test.
Then, 3000 slices are used for training/evaluation, and 500 slices are used for test. These MR images are fully sampled images, so we make undersampled images by a randomly subsampling k-space lines. The acceleration factor is four, and autocalibration signal (ACS) region contains 4\% of k-space lines. Each slice is normalized by standard deviation of the magnitude of each slice. To handle complex values of data, we concatenate real and imaginary values along the channel dimension. Each slice has different size, so we use only single batch. The images are center-cropped to the size of $320\times 320$, and then the peak signal-to-noise ratio (PSNR) and structural similarity index (SSIM) values are calculated.
Here, the PSNR and SSIM are defined as
\begin{equation}\label{eq:psnr}
{\rm PSNR} = 20 \log_{10} \left(\frac{n\|x^*\|_\infty}{\| x- x^*\|_2}\right)
\end{equation}
and
\begin{equation}\label{eq:ssim}
{\rm SSIM} = \frac{(2\mu_{x}\mu_{x^*}+c_1)(2\sigma_{x x^*}+c_2)}{(\mu_{x}^2+\mu_{x^*}^2+c_1)(\sigma_{x}^2+\sigma_{x^*}^2+c_2)},
\end{equation}
where $x$ and $x^*$ denote the reconstructed images and ground truth, respectively,
  $n$ is the number of pixels,  $\mu_{x}$ is an average of $x$, $\sigma_{x}^2$ is a variance of $x$ and $\sigma_{x x^*}$ is a covariance of $x$ and $x^*$,
  and $c_1, c_2$ are two variables to stabilize the division.

\begin{figure}[!h]
  \centerline{\includegraphics[width=0.7\textwidth]{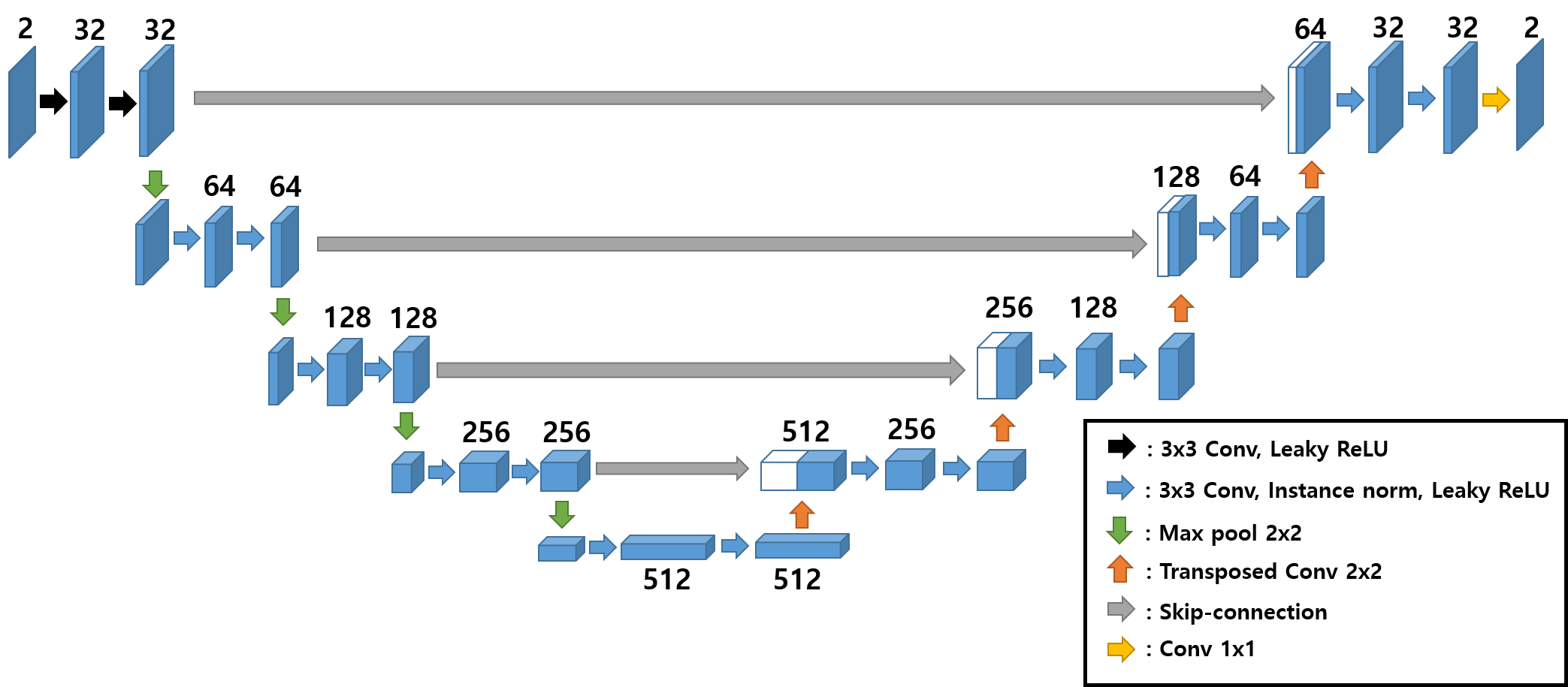}}
  \centerline{\mbox{(a)}}  
    \centerline{\includegraphics[width=0.4\textwidth]{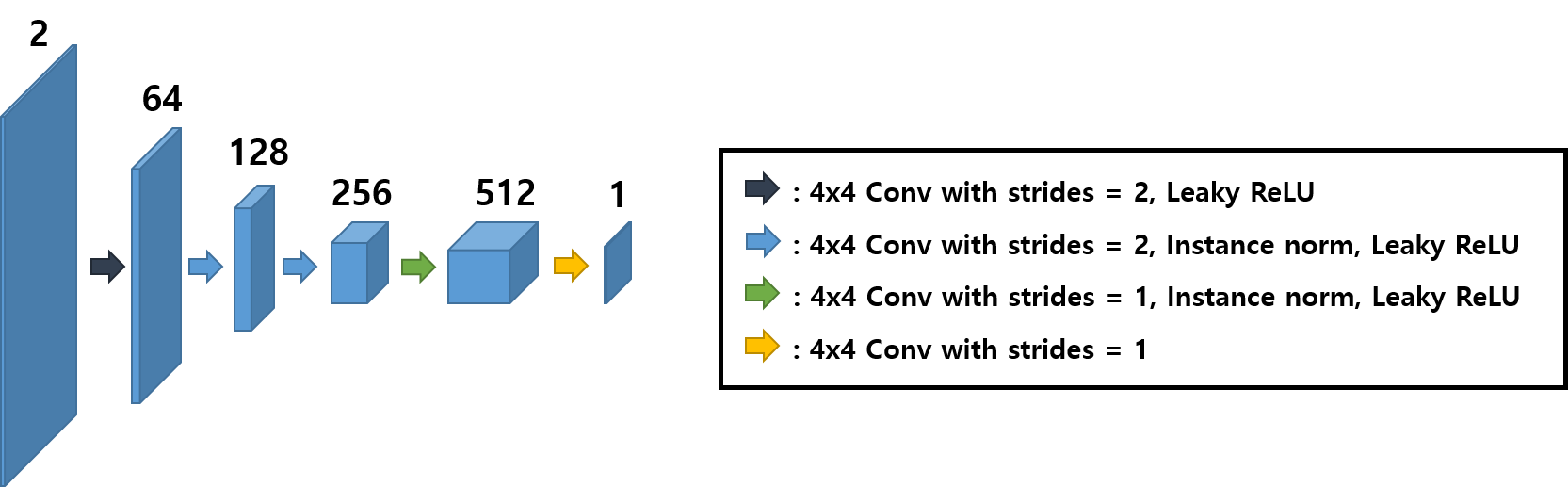}}
  \centerline{\mbox{(b)}}  
  \caption{Proposed network architectures for (a) generator and (b) discriminator for accelerated MRI.}
  \label{fig:mr_net}
\end{figure}

\begin{figure}[!hbt]
  \center{
   \includegraphics[width=0.8\textwidth]{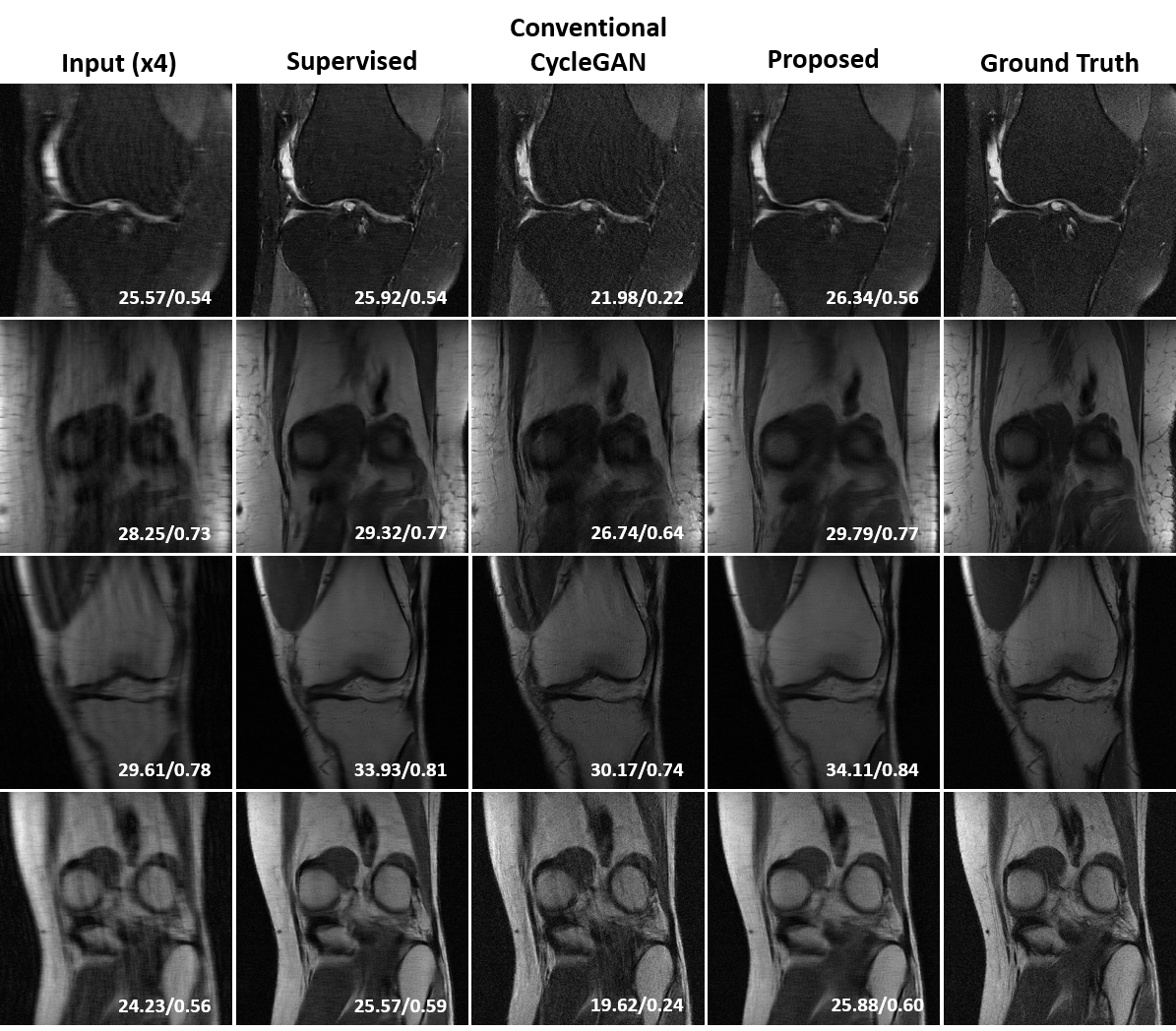}
}\vspace{-0.2cm}
  \caption{Unsupervised learning results for accelerated MRI using proposed cycleGAN. The values in the corners are PSNR/SSIM values for each image. }
  \label{fig:mr_result}
\end{figure}

We use U-Net generator to reconstruct fully sampled MR images from undersampled MR images as shown in Fig.~\ref{fig:mr_net}(a). Our generator consists of $3\times 3$ convolution, instance normalization, and leaky ReLU operation. Also, there are skip-connection and pooling layers. At the last convolution layer, we do not use any operation.
We use PatchGAN architecture \cite{isola2017image} as our discriminator network, so the discriminator classifies inputs at patch scales. The discriminator also consists of convolution layer, instance normalization, and leaky ReLU operation as shown in Fig.~\ref{fig:mr_net}(b). 
We use the WGAN-GP to impose 1-Lipschitz constraint for the discriminator.
We use Adam optimizer to train our network, with momentum parameters $\beta_1=0.5, \beta_2=0.9$, and learning rate of 0.0001. The discriminator is updated 5 times for every generator updates. We use batch size of 1, and trained our network during 100 epochs. Our code was implemented by TensorFlow.

We implement other methods to compare the result of proposed method, but we do not compare with compressed sensing based methods because they are time-consuming and burdensome to find hyperparameters. Compared methods are implemented with the same generator architecture as the proposed method.
The reconstruction results in Fig.~\ref{fig:mr_result} and quantitative comparison results for all test sets in Table~\ref{tbl:MRresult} clearly show that
our OT-cycleGAN architecture with a single generator
successfully recovered fine details without matched references.
Moreover, the performance is even better than that of the standard cycleGAN,
and it is comparable with that of supervised learning.

\begin{table}[!thb]
   \centering
   \caption{Quantitative comparison for various algorithms on 500 test sets of fastMRI data.}
   \label{tbl:average metric}
   \resizebox{0.9\textwidth}{!}{
      \begin{tabular}{c | c c c c}
         \hline
         \ {Metric}         & Input ($\times 4$)      & Supervised Learning      & Conventional CycleGAN      & Proposed Method   \\ \hline\hline
         \ PSNR (dB) (\ref{eq:psnr})         & 26.81               & 27.96               & 25.01               & \textbf{28.17}   \\
         \ SSIM (\ref{eq:ssim})        & 0.6419               & 0.6419               & 0.4689               & \textbf{0.6550}   \\ \hline
      \end{tabular}
   }
	\label{tbl:MRresult}
\end{table}

\subsection{Single-molecule Localization Microscopy}
\label{sec:microscopy}

Single-molecule localization microscopy methods, such as STORM 
\cite{rust2006sub} and (F)PALM \cite{hess2006ultra, betzig2006imaging}, utilize
sparse activation of photo-switchable fluorescent probes in both temporal and
spatial domains. Each activated probe is assumed as an ideal point source 
so that one can achieve sub-pixel accuracy on the order of tens of nanometers for the estimated location of each probe \cite{rust2006sub,betzig2006imaging, smith2010fast,henriques2010quickpalm,parthasarathy2012rapid}. In general, reconstruction of sub-cellular structures relies on numerous localized probes, which leads to relatively long acquisition time.
To reduce the acquisition time,  high density imaging with sparsity regularized deconvolution
approaches such as CSSTORM \cite{zhu2012faster} (Compressed sensing STORM)  and deconvolution-STORM (deconSTORM)\cite{mukamel2012statistical} have been investigated.
Recently,  CNN approaches such as DeepSTORM \cite{deepSTORM2018} have been extensively
studied as fast and high-performance alternatives. 
Unfortunately, the existing CNN approaches usually require matched high-resolution images for supervised training, which is not practical.

To address the problem,   we are interested in developing unsupervised learning approach for super-resolution microscopy.
Mathematically, a blurred measurement
 can be described as
\begin{eqnarray}\label{eq:forward}
y &=&h\ast x + w \ ,
\end{eqnarray}
where $h$ is the PSF.
Here, we consider the blind deconvolution problem 
where both the unknown PSF $h$  and the image $x$ should be estimated.
This leads to the OT-cycleGAN corresponding to Table~\ref{tbl:algorithm}(b).
The schematic diagram of the corresponding cycleGAN architecture is illustrated in Fig.~\ref{fig:deconv_arch}.
The leftmost images are samples from low-resolution or high-resolution images.
Given the samples, the generator network and the linear blur operator generate  fake high-resolution image and blurred image.
Two pairs of images, the fake  and real high-resolution images, and the synthetically blurred and the low-resolution measured images, are then passed to discriminators.
  
In contrast to the standard cycleGAN approaches that require two deep generators, 
the proposed cycleGAN approach needs  only a single deep generator, and the blur generator is implemented using
a linear convolution layer corresponding to the unknown PSF. It is important to note that unlike the accelerated MRI problem in the previous section,
we still need two discriminators, since both linear and deep generators should be estimated.
Still the simplicity of the proposed cycleGAN architecture compared to the standard cycleGAN  significantly improves the robustness of network training.

\begin{figure}[h!]
\center{\includegraphics[width=0.7\textwidth]{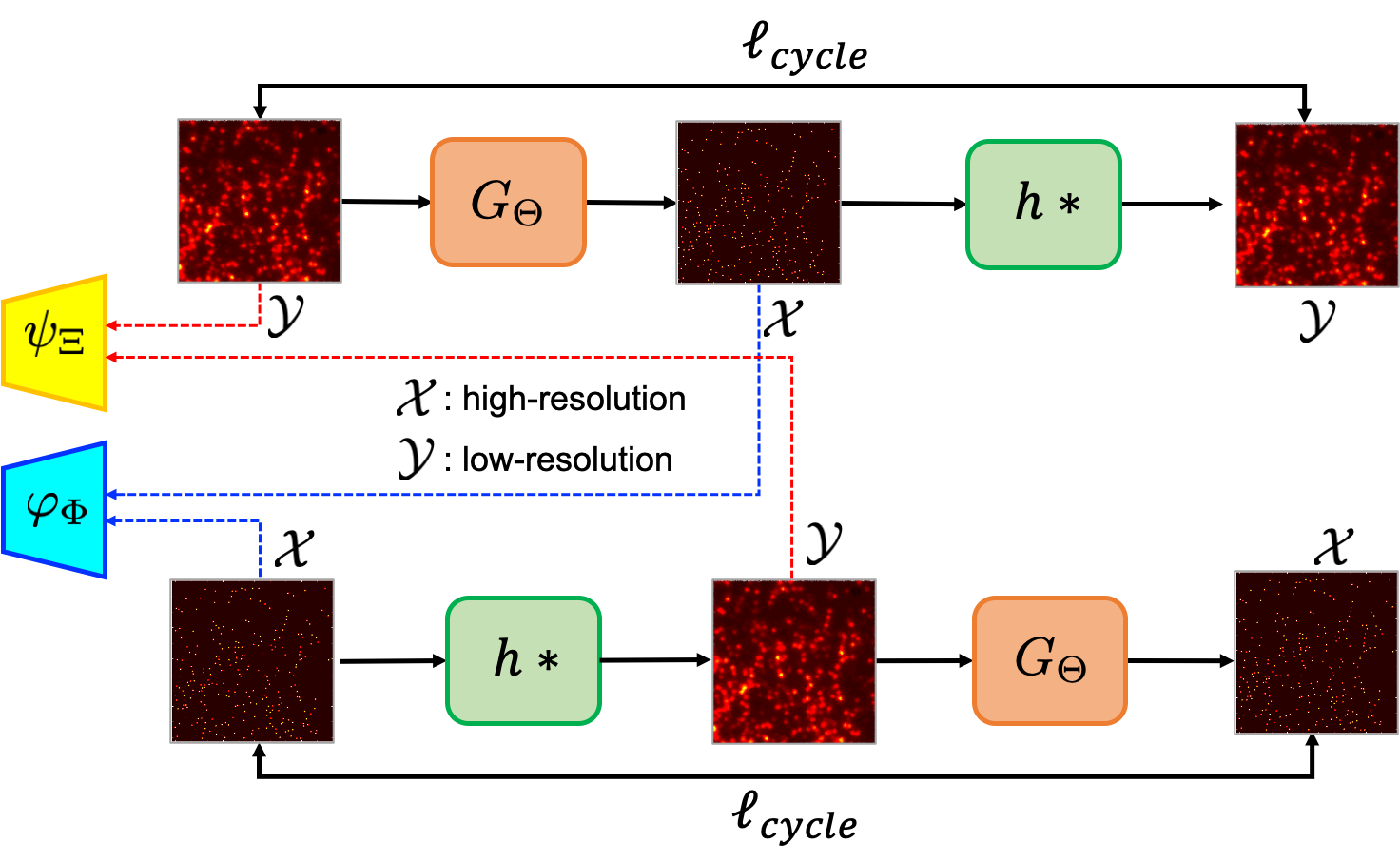}}
\caption{Proposed cycleGAN architecture with a blur kernel for deconvolution microscopy.
This corresponds to OT-CycleGAN in Fig.~\ref{fig:cycleGAN}(b) since the blurring of images in microscopy is modeled as a convolution with a linear unknown kernel. Therefore, one of the generators can be replaced by a simple linear layer.
}

\label{fig:deconv_arch}
\end{figure}

The architecture of the high resolution image generator $G_{\Theta}$ from the low-resolution image is a modified U-net \cite{RonnebergerFB15} as shown in Fig~\ref{fig:deconv_net}(a).
Our U-net consists of two parts: encoder and decoder. For both encoder and decoder, {batch normalization} \cite{ioffe2015batch} and ReLU layers follow every convolution layer except the lasted one. At the lasted layer, only batch normalization layer follows the convolution layer.
Throughout the whole network, the convolutional kernel size is $3\times3$ with no stride. A max pooling layer exist at the encoder part with kernel size $2\times2$ and stride 2. 
The number of the channel increases and becomes the maximum at the end of encoder part. In this paper, the maximum number of the channel is 300.

For the second generator, a single 2D convolution layer is used. This layer is to the model of 2D PSF to generate
blurry low-resolution image from high resolution input. The kernel size of the single 2D convolution layer is set to $10\times10$.
As for the discriminators, we used simple convolution-based architecture as shown in Fig \ref{fig:deconv_net}(b). 
It has three layer-blocks. Each block consists of {2D convolution} - {batch normalization} - {ReLU}. 
We use the WGAN-GP to impose 1-Lipschitz constraint for the discriminators.
Throughout the whole network, the convolutional kernel size is $5\times5$ with stride 3. 
After all three blocks, fully
connected layer followed by a sigmoid function is added to generator scalar value decision between 0 and 1. 

\begin{figure}[!h]
	\centerline{\includegraphics[width=0.7\textwidth]{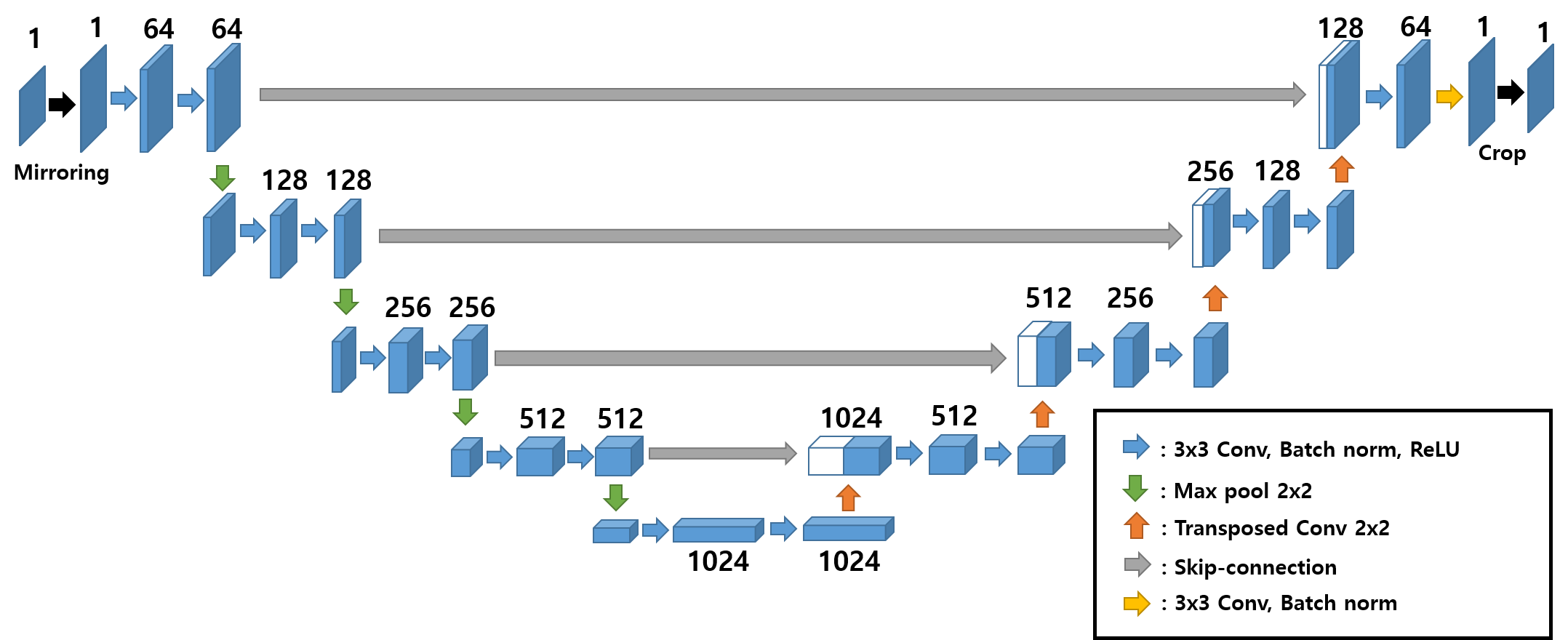}}
	\centerline{\mbox{(a)}}  
	\centerline{\includegraphics[width=0.4\textwidth]{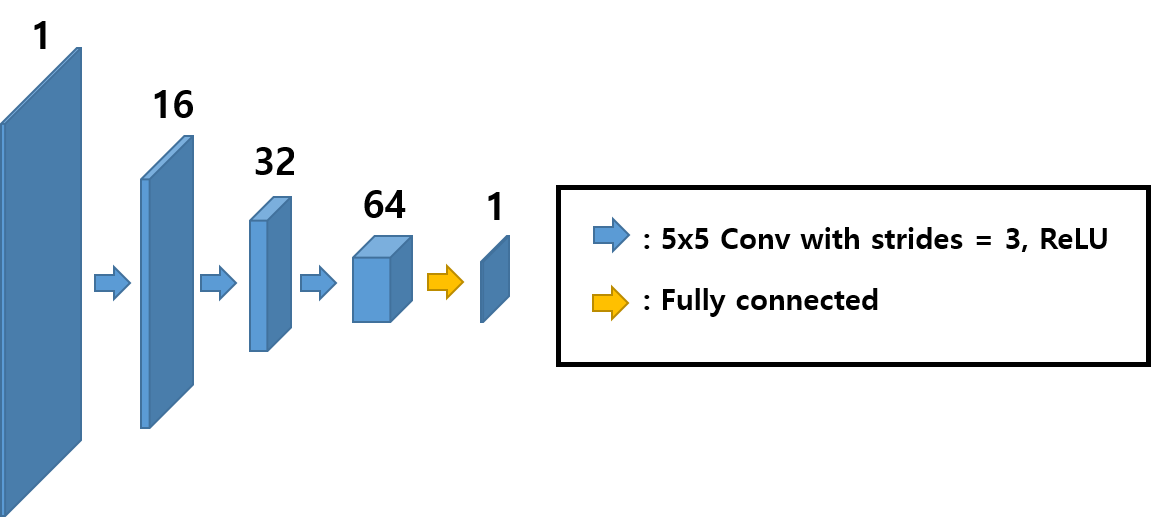}}
	\centerline{\mbox{(b)}}  
	\caption{Proposed network architectures for (a) generator and (b) discriminator for high-resolution image.}
	\label{fig:deconv_net}
\end{figure}

\begin{figure*}[h]
\centering\includegraphics[width=0.95\textwidth]{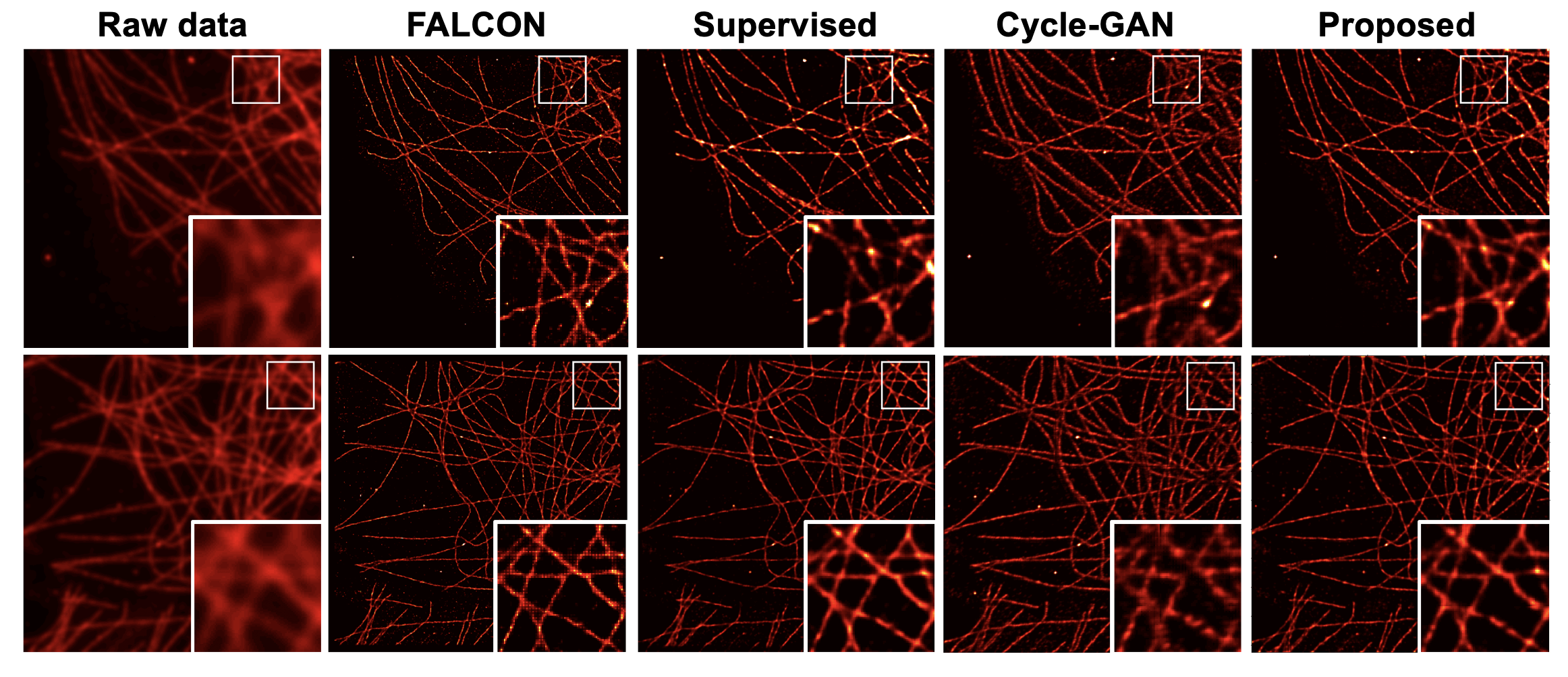}
\caption{Reconstruction results by various methods.  (First column) low resolution raw data from SMLM, (second column) FALCON reconstruction, (third column) supervised learning results, (fourth) standard cycleGAN,  and (the last column)  the proposed OT-cycleGAN.
The image within white box is magnified by 2 times.
}
\label{fig:deconv_result}
\end{figure*}

For our experiments, three real single molecule localization microscope (SMLM) image sets obtained from $\alpha$-tubulin subunits of microtubules are used. Among the three data sets, the first data set composed of 500 frames of  $128\times 128$ image was used for training.
The second data sets was composed of 1000 frames of  $128\times 128$ measurement,
and the third data set was composed of 2000 frame of $128\times 128$ measurement. These two data set were used for test.
Using data augmentation by vertical and horizontal flipping, we increased the training data set size by four times. Among them,  1800 frames were used for training, and the other 200 frames for validation.
From these low resolution measurements, high resolution images  of $256\times 256$ 
were reconstructed using neural networks.
To generate the larger size images, the low resolution images are first upsampled by four times  (i.e. $2\times 2$ upsampling) to be used
as a network input.
To overcome GPU memory limitation, the input images were divided into $128\times128$ size patches. We normalized the scale of the whole input images to [0,1]. We employed  Adam optimizer \cite{kingma2014adam} and the learning rate was set to 0.0001 (constant value). The number of epoch was 50, and 1800 optimization were done per each epoch.
Due to a lack of matched label data, 
we use the random point set images as high resolution data set in our unsupervised training, and our neural
network was trained to follow the distribution of random point sets.

As for comparison,  we employed a state-of-the-art method called FAst Localization algorithm based on a CONtinuous-space formulation
(FALCON) \cite{Min2014}
that encompass classical PLS, and DeepSTORM \cite{deepSTORM2018} that uses the U-Net architecture in a supervised learning framework.
For supervised training, we used the high resolution images from FALCON as label data.

Fig.~\ref{fig:deconv_result} shows the comparison results for two test data set.
To generate the images, each frame of the SMLM was processed with the super-resolution algorithm, after which
all the temporal frames are averaged.
As shown in  Fig.~\ref{fig:deconv_result}, all methods provide the comparable reconstruction results while FALCON image looks sharper
since it is based on the centroid coordinate estimation by Talyor series expansion, rather than image deconvolution.
However, the proposed method provides reconstructed images comparable with supervised learning approach, and superior to the standard cycleGAN.
To quantify the resolution improvement, we use the 
Fourier ring correlation (FRC) \cite{Koho2019}, which is a widely used criterion for resolution estimation.
Interestingly, the FRC resolution shown in Table \ref{tbl:average metric 2D} showed that the proposed method
provides the best FRC resolution.

\begin{table}[!thb]
   \centering
   \caption{Quantitative comparison of reconstruction resolution using FRC by various algorithms on the two test data sets. The unit of the FRC is nm and the value shows average and standard deviation.}
   \label{tbl:average metric 2D}
   \resizebox{0.85\textwidth}{!}{
      \begin{tabular}{c | c c c c}
         \hline
         \ {Test data}		&FALCON					& Supervised Learning	 & Standard CycleGAN & Proposed\\ \hline\hline
         \ Test set 1	&${35.30 \pm 6.16}$ (nm)	&${35.24 \pm 6.32}$ (nm)	& ${46.05 \pm 5.78}$ (nm)	 & ${\textbf{30.33} \pm 4.42}$ (nm) \\ 
         \ Test set 2	&${32.29 \pm 5.19}$ (nm)	&${32.68 \pm 4.93}$	(nm) & ${36.77 \pm 2.01}$ (nm)		&${\textbf{27.82} \pm 4.16}$ (nm)\\\hline
      \end{tabular}
      }
	\label{tbl:2Dresult}
\end{table}

\subsection{Image Denoising in X-Ray CT}
\label{sec:ct}

X-ray CT retains a potential risk of cancer, induced by radiation exposure. Althought the low dose CT (LDCT) techniques can reduce the radiation
risk,  it also sacrifies the signal-to-noise ratio (SNR) in the obtained images compared to standard-dose CT (SDCT). Therefore, extensive studies have been carried out  to achieve reduced noise level in LDCT images. While model-based iterative reconstruction (MBIR) methods \cite{beister2012iter, ramani2012split, sidky2008cone, chen2008piccs} have been developed to address the problem,  the MBIR methods usually suffer from long reconstruction time due to iterative procedure. 
Recently, deep learning approaches demonstrated improved performance compared to the conventional methods \cite{kang2017deep, kang2018framelet, chen2017residual}. 
Unfortunately, it is difficult to get paired LDCT and SDCT data in clinical practice. Therefore, the approach based on a unsupervised learning is necessary for practical applications. 

Since the noise patterns in LDCT contains many streaking  and nonlinear phenomenon dependent artifacts,
it is difficult to model the noise as simple Gaussian noise distribution.  Therefore,
rather than modeling the noise with specific closed-form statistics,  we assume that
the  LDCT image is given by the nonlinear mapping
from the SDCT image. This is modeled by another neural network
$\Hc_\Upsilon$.   The resulting OT-cycleGAN then has loss functions given in Table~\ref{tbl:algorithm}(d).

The schematic diagram of the resulting cycleGAN architecture is illustrated in Fig.~\ref{fig:ct_arch}.
The leftmost images are  samples from low-dose or full-dose CT images.
Given the samples, two generator networks generate fake high-does and low-dose image, respectively, and then return them
to the original images.
The fake images and real samples are passed to two discriminators (the Kantorovich potentials).
In fact, this structure is same as the standard GAN in Fig.~\ref{fig:cycleGAN}(a) and Table~\ref{tbl:algorithm}(a).

\begin{figure}[!h]
	\center{
		\includegraphics[width=0.8\textwidth]{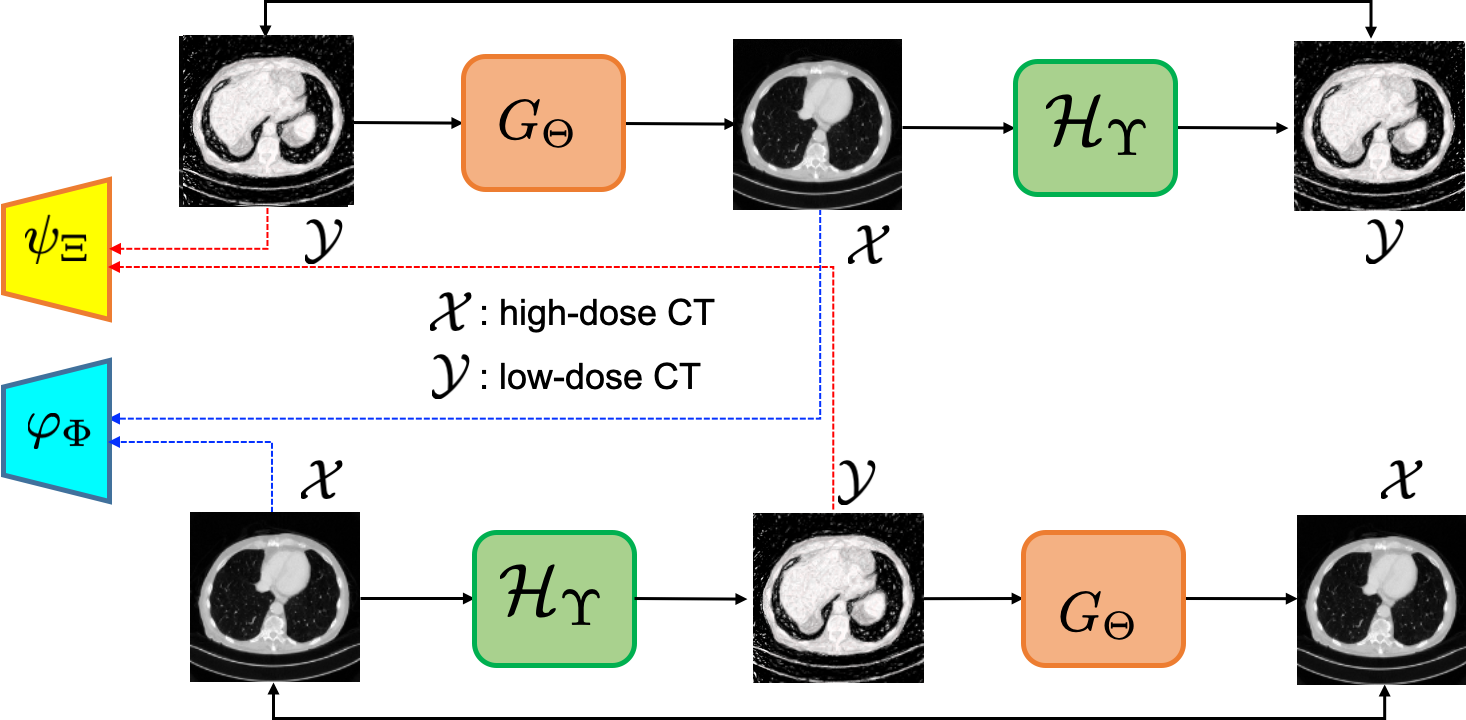}
	}
	\caption{Proposed OT-cycleGAN architecture for CT image denoising. 
		This corresponds to OT-CycleGAN in Fig.~\ref{fig:cycleGAN}(d) since the forward operator is not known. Therefore, we need
		two pairs of generators and discriminators implemented using neural networks.
}
	\label{fig:ct_arch}
\end{figure}

The generators $G_\Theta$ and $\Hc_\Upsilon$ are implemented using the U-net structure as shown in Fig.~\ref{fig:ct_net}(a). They have 3 stages of pooling. The feature map at each stage is extracted by convolutional blocks, which consist of $3\times3$ convolution, instance normalization, and ReLU. Also, there is skip connection for every stage.  The $128\times128$ patch images cropped from original image are used as inputs for the generator.
For the discriminators, the architecture of PatchGAN as shown in Fig.~\ref{fig:ct_net}(b) was
used. The network gets $128\times128$ patch image as an input, and decides whether they are real or generated from the generator. The discriminator consists of $4\times4$ convolution, Instance normalization and ReLU operation.
In this experiment, we use both WGAN-GP to impose 1-Lipschitz constraint for the discriminators.
Adam optimizer with the batch size of 20 was used for training. 

\begin{figure}[!h]
	\centerline{\includegraphics[width=0.7\textwidth]{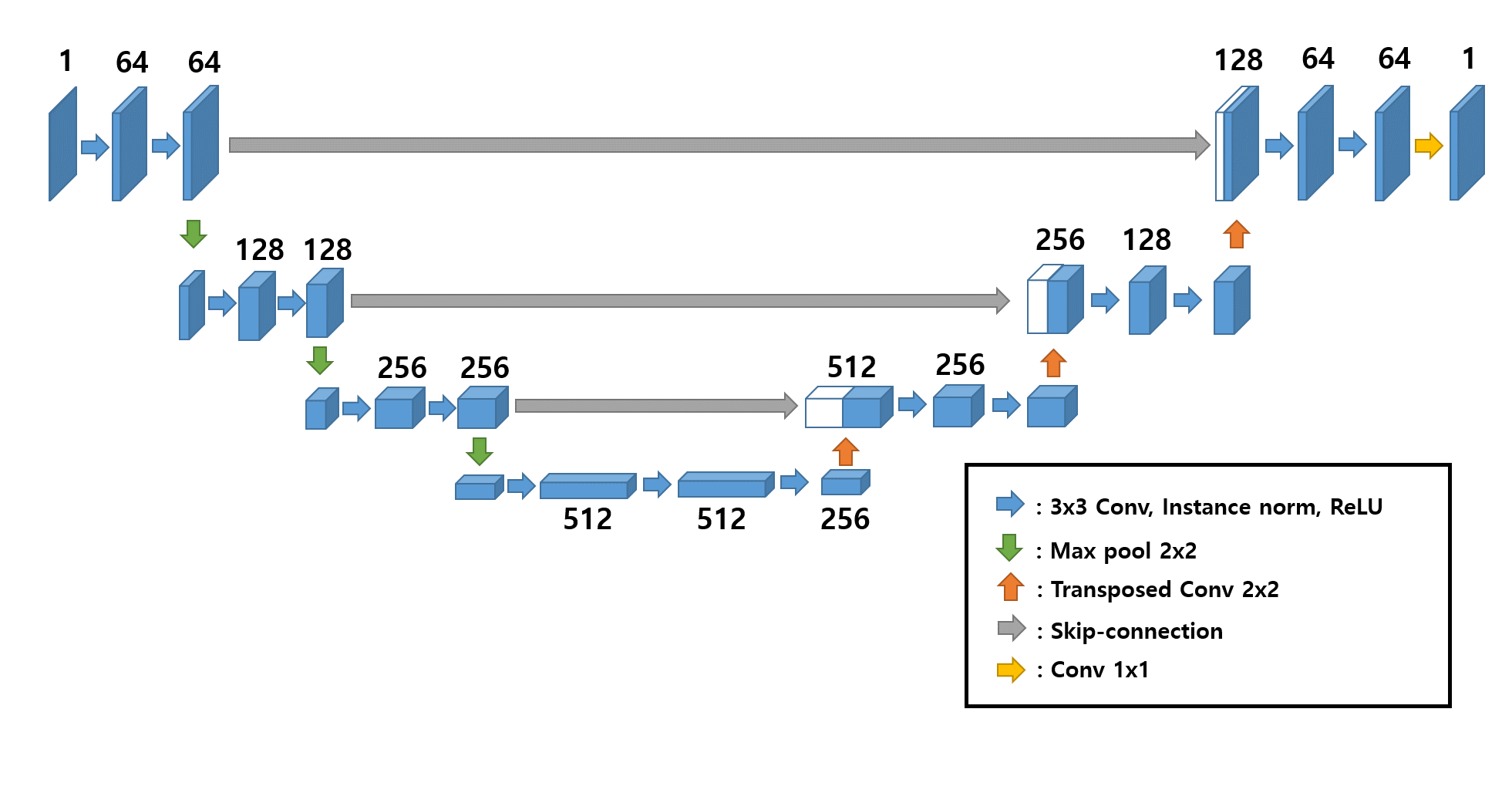}}
	\centerline{\mbox{(a)}}  
	\centerline{\includegraphics[width=0.4\textwidth]{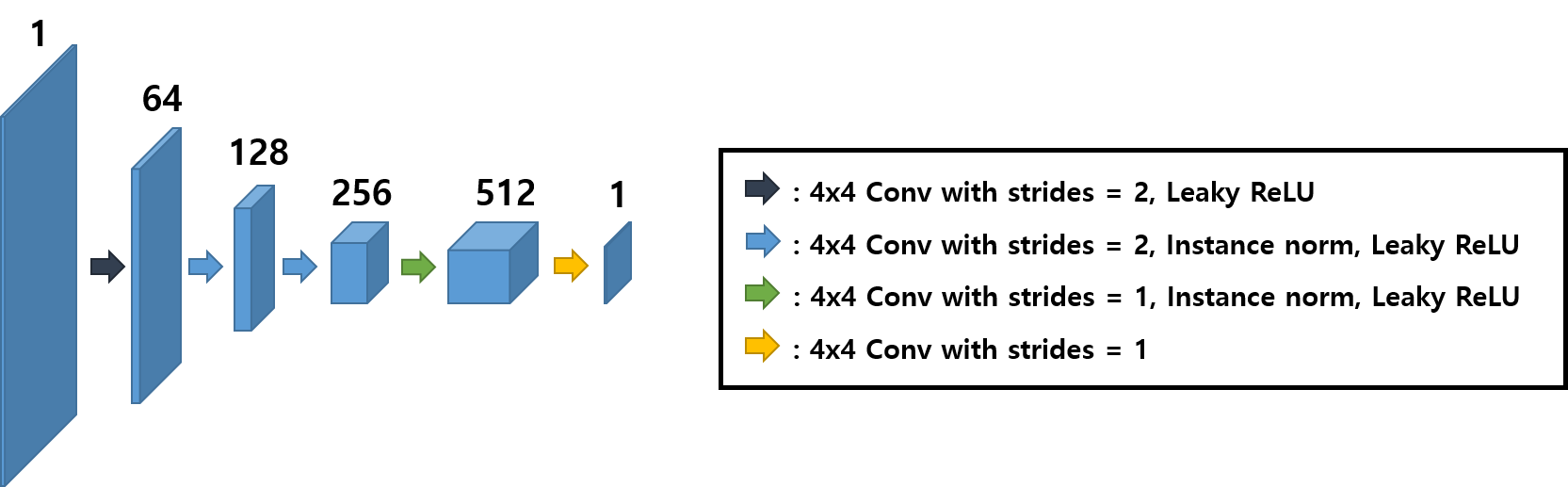}}
	\centerline{\mbox{(b)}}  
	\caption{Proposed network architectures for (a) generator and (b) discriminator for CT denoising.}
	\label{fig:ct_net}
\end{figure}

The abdominal CT dataset provided in Low Dose CT Grand Challenge of American Association of Physicists in Medicine (AAPM) was used for the experiments.   Total 3000 slices LDCT and SDCT data are used for training,  and 500 slices for validation. Another 421 slices of LDCT and SDCT data  are chosen to test the model.
The dataset is paired, but we do not utilize paired information during training. The paired information is only used to evaluate metric during test.
All models are trained for 400 epoch, and the output metric from each model is evaluated.
In training procedure for the cycleGAN, LDCT and SDCT data sets are  unpaired. For the supervised learning,
the same generator architecture was used and the paired training data are used.
In the evaluation procedure, the ground-truth SDCT data are used to obtain PSNR and SSIM to quantify the performance.

\begin{figure}[!h]
	\center{
		\includegraphics[width=0.6\textwidth]{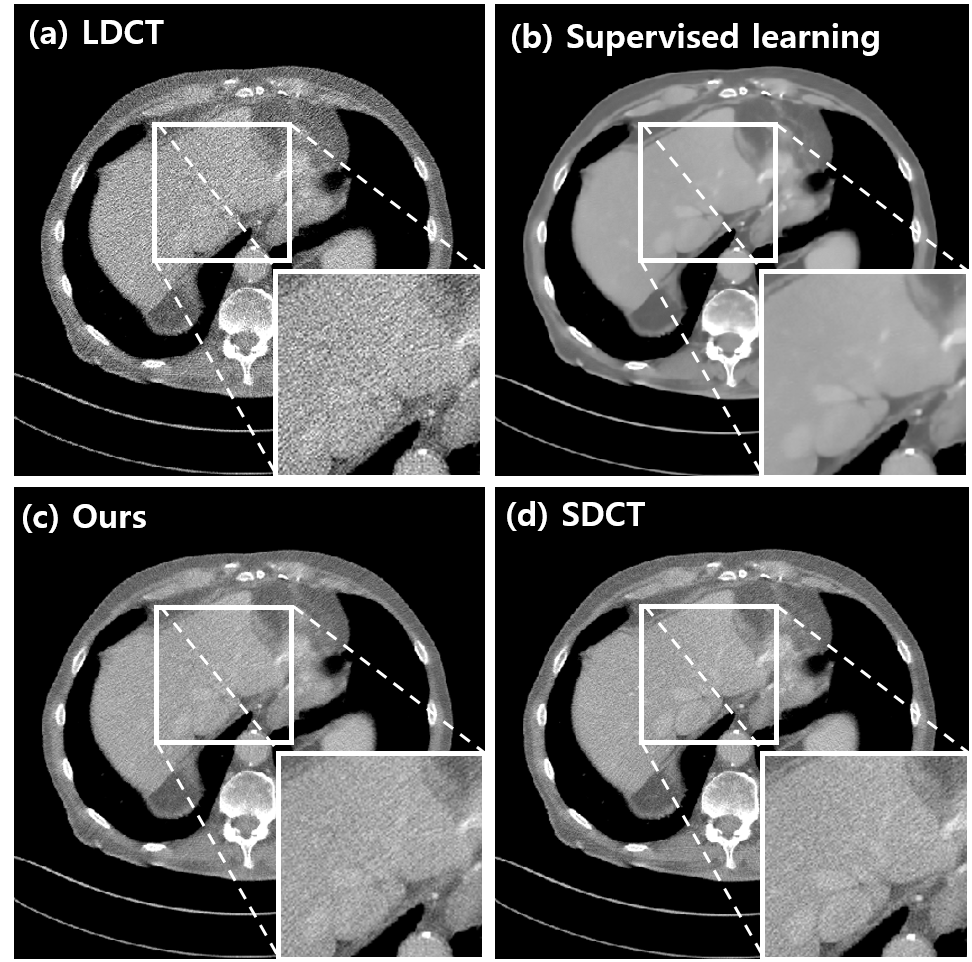}
	}
	\caption{Qualitative comparison of denoising performance. LDCT image, and denoising results
	by supervised learning, our OT-cycleGAN, and SDCT image. CT images are displayed with (-1000, 400)[HU] window.}
	\label{fig:ct_result}
\end{figure}

Table~\ref{tbl:ct_result} and Fig.\ref{fig:ct_result} demonstrates that the denoising was successfully achieved by the proposed OT-cycleGAN,
which is equivalent to the standard cycleGAN in this case. Also, it is remarkable that the performance of the cycleGAN is comparable with that of supervised learning, even though any paired data was used for training. 

\begin{table}[!thb]
	\centering
	\caption{Average quantitative performance comparison using 421 test sets of CT data}
	\resizebox{0.6\textwidth}{!}{
		\begin{tabular}{c | c c c c}
			\hline
			\ {Metric}         & Input       & Supervised Learning   & Proposed          \\ \hline\hline
			\ PSNR (dB)  (\ref{eq:psnr})          & 31.30      & 38.66              & 38.20                  			    \\
			\ SSIM     (\ref{eq:ssim})           & 0.81   & 0.93              & 0.92               					\\ \hline
		\end{tabular}
	}
	\label{tbl:ct_result}
\end{table}

\section{Conclusions}

In this paper, we presented a novel OT-cycleGAN framework that can be used for various inverse problems.
Specifically, the proposed OT-cycleGAN was obtained from Kantorovich dual OT problems, where  
a novel PLS cost with the data consistency and deep learning prior is used as a transportation cost.
As proofs of concept, we designed  three distinct OT-cycleGAN architectures for accelerated MRI, super-resolution
microscopy,  low-dose CT reconstruction problems, providing accurate reconstruction
results without any matched reference data.
Given the generality of our design principle, we believe that our method can be an important platform for unsupervised
learning for inverse problems.

\medskip

\bibliographystyle{siamplain}
\bibliography{refs,submit_bib,strings}

\end{document}